%% file: sketchedOL.tex
\documentclass{article}

 \PassOptionsToPackage{numbers}{natbib}
\usepackage[final]{nips_2016}

\usepackage[utf8]{inputenc} 
\usepackage[T1]{fontenc}    
\usepackage{hyperref}       

\usepackage{url}            
\usepackage{booktabs}       
\usepackage{nicefrac}       
\usepackage{microtype}      

\usepackage{amsmath,amsthm,amsfonts,amssymb,mathrsfs}
\usepackage{algorithm}
\usepackage{float}
\usepackage{algorithmicx}
\usepackage{algpseudocode}
\usepackage{subfigure} 
\usepackage{graphicx}
\usepackage{mathtools}
\usepackage{float}
\usepackage{xspace}
\usepackage{multirow}
\usepackage{wrapfig}

\newcommand{\bz}{\boldsymbol{z}}
\newcommand{\bx}{\boldsymbol{x}}

\newcommand{\bu}{\boldsymbol{u}}

\newcommand{\loss}{\ell}

\newcommand{\bg}{\boldsymbol{g}}
\newcommand{\ba}{\boldsymbol{a}}
\newcommand{\be}{\boldsymbol{e}}
\newcommand{\bq}{\boldsymbol{q}}
\newcommand{\bp}{\boldsymbol{p}}

\newcommand{\bbw}{\bar{\boldsymbol{w}}}
\newcommand{\bbu}{\bar{\boldsymbol{u}}}
\newcommand{\bw}{\boldsymbol{w}}

\newcommand{\bzero}{\boldsymbol{0}}
\newcommand{\balpha}{\boldsymbol{\alpha}}
\newcommand{\bbeta}{\boldsymbol{\beta}}
\newcommand{\bdelta}{\boldsymbol{\delta}}

\newcommand{\scK}{\mathcal{K}}

\newcommand{\Ahat}{\wh{A}}

\newcommand{\bghat}{\wh{\boldsymbol{g}}}

\newcommand{\bxhat}{\wh{\bx}}

\DeclareMathOperator*{\argmin}{argmin}

\DeclareMathOperator*{\range}{range}
\DeclareMathOperator*{\mydet}{det_{+}}

\newcommand{\field}[1]{\mathbb{#1}}

\newcommand{\R}{\field{R}}

\newcommand{\E}{\field{E}}

\newcommand{\btheta}{\boldsymbol{\theta}}
\newcommand{\blambda}{\boldsymbol{\lambda}}

\newcommand{\theset}[2]{ \left\{ {#1} \,:\, {#2} \right\} }
\newcommand{\inner}[1]{ \left\langle {#1} \right\rangle }

\newcommand{\eye}[1]{ \boldsymbol{I}_{#1} }
\newcommand{\norm}[1]{\left\|{#1}\right\|}
\newcommand{\trace}[1]{\textsc{tr}({#1})}
\newcommand{\diag}[1]{\mathrm{diag}\!\left\{{#1}\right\}}

\newcommand{\defeq}{\stackrel{\rm def}{=}}
\newcommand{\sgn}{\mbox{\sc sgn}}

\newcommand{\scO}{\mathcal{O}}

\renewcommand{\ss}{\subseteq}
\newcommand{\wh}{\widehat}
\newcommand{\wt}{\widetilde}

\newtheorem{lemma}{Lemma}
\newtheorem{theorem}{Theorem}
\newtheorem{cor}{Corollary}

\newtheorem{prop}{Proposition}
\newtheorem{assumption}{Assumption}

\newcommand{\bb}{\boldsymbol{b}}

\algnewcommand\Internal{\item[{\textbf{Internal State:}}]}
\algnewcommand\Update[1]{\item[{\textbf{SketchUpdate({#1})}}]}
\algnewcommand\Init[1]{\item[{\textbf{SketchInit({#1})}}]}

\newcommand{\order}{\ensuremath{\mathcal{O}}}

\newcommand{\alg}{\textsc{OjaSENSE}\xspace}
\newcommand{\adagrad}{\textsc{AdaGrad}\xspace}
\newcommand{\sign}{\ensuremath{\mbox{sign}}}

\renewcommand{\alg}{SON\xspace}
\newcommand{\alglong}{Sketched Online Newton\xspace}
\newcommand{\ojaalg}{Oja-SON\xspace}

\newcommand{\fdalg}{FD-SON\xspace}

\newcommand{\specialcell}[2][c]{\begin{tabular}[#1]{@{}c@{}}#2\end{tabular}}

\newcommand{\scale}{\ensuremath{\gamma}}



\title{Efficient Second Order Online Learning by Sketching}

\author{
  Haipeng Luo \\
  Princeton University, Princeton, NJ USA \\
  \nolinkurl{haipengl@cs.princeton.edu} \\
  \And
  Alekh Agarwal \\
  Microsoft Research, New York, NY USA \\
  \nolinkurl{alekha@microsoft.com} \\
  \And
  Nicol\`o Cesa-Bianchi \\
  Universit\`a degli Studi di Milano, Italy \\
  \nolinkurl{nicolo.cesa-bianchi@unimi.it} \\
  \And
  John Langford \\
  Microsoft Research, New York, NY USA \\
  \nolinkurl{jcl@microsoft.com} \\  
}

\begin{document}

\maketitle

\begin{abstract}
We propose \alglong (\alg), an online second order learning algorithm
that enjoys substantially improved regret guarantees for
ill-conditioned data. \alg is an enhanced version of the Online Newton
Step, which, via sketching techniques enjoys a running time linear in
the dimension and sketch size.  We further develop sparse forms of the
sketching methods (such as Oja's rule), making the computation linear
in the sparsity of features. Together, the algorithm eliminates all
computational obstacles in previous second order online learning
approaches.
\end{abstract}

\section{Introduction}
\label{sec:intro}

Online learning methods are highly successful at rapidly reducing the
test error on large, high-dimensional datasets. First order methods
are particularly attractive in such problems as they typically enjoy
computational complexity linear in the input size.  However, the
convergence of these methods crucially depends on the geometry of the
data; for instance, running the same algorithm on a rotated set of
examples can return vastly inferior results. See
Fig.~\ref{fig:synthetic} for an illustration.

Second order algorithms such as Online Newton Step~\citep{HazanAgKa07}
have the attractive property of being invariant to linear
transformations of the data, but typically require space and update
time quadratic in the number of dimensions. Furthermore, the
dependence on dimension is not improved even if the examples are
sparse. These issues lead to the key question in our work: \emph{Can
  we develop (approximately) second order online learning algorithms
  with efficient updates?}  We show that the answer is ``yes'' by
developing efficient sketched second order methods with regret
guarantees.  Specifically, the three main contributions of this work
are:

\begin{wrapfigure}{R}{0.5\textwidth}
\centering
  \includegraphics[width=.35\textwidth]{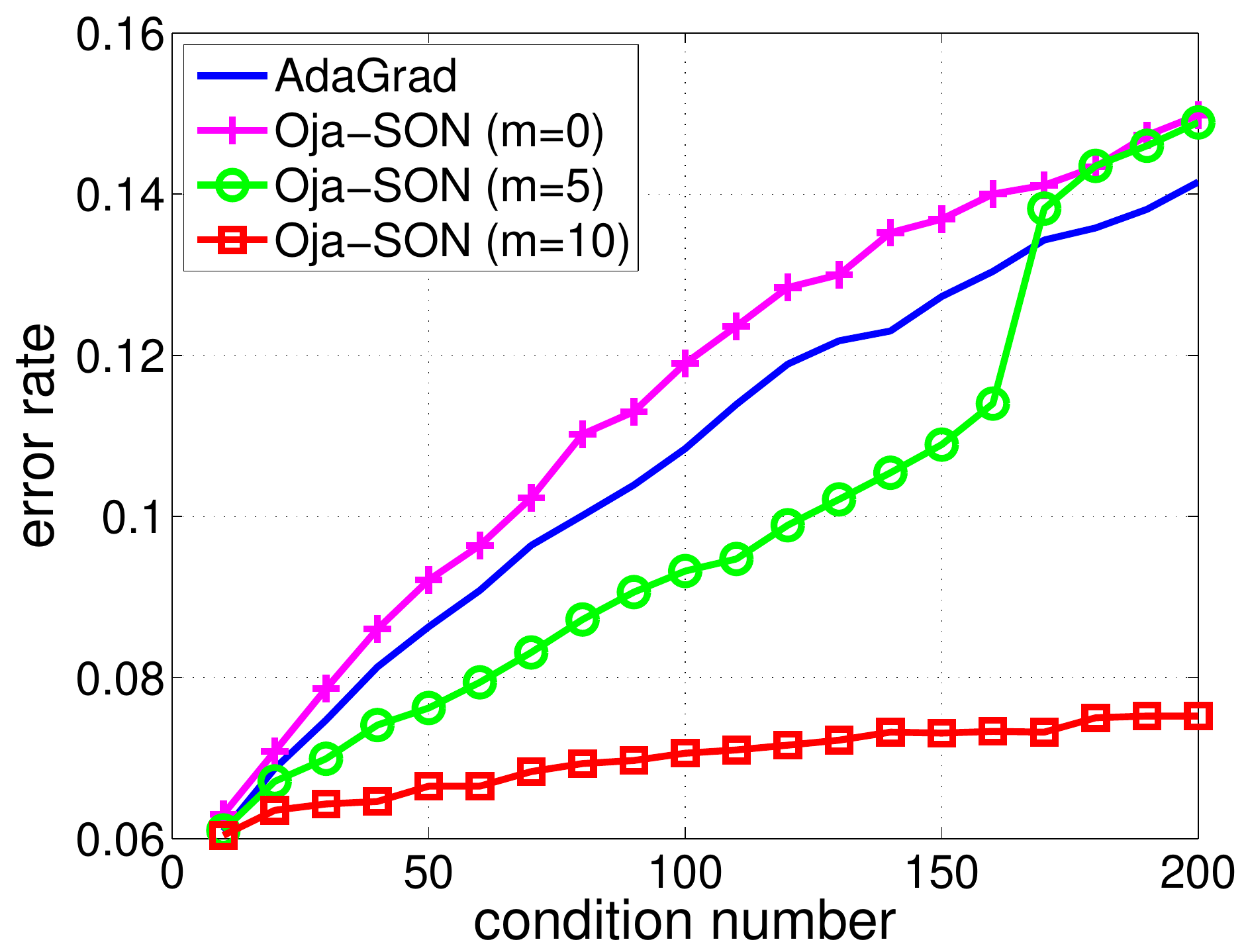}
  \caption{Error rate of \alg using Oja's sketch, and \adagrad on a
    synthetic ill-conditioned problem. $m$ is the sketch size ($m = 0$
    is Online Gradient, $m = d$ resembles Online Newton). \alg is
    nearly invariant to condition number for
    $m=10$.}\label{fig:synthetic}
\end{wrapfigure}

\paragraph{1. Invariant learning setting and optimal algorithms (Section~\ref{sec:setup}).}
The typical online regret minimization setting evaluates against a
benchmark that is bounded in some fixed norm (such as the
$\ell_2$-norm), implicitly putting the problem in a nice geometry.
However, if all the features are scaled down, it is desirable to
compare with accordingly larger weights, which is precluded by an
apriori fixed norm bound.  We study an invariant learning setting
similar to the paper~\citep{RossMiLa13} which compares the learner to
a benchmark only constrained to generate bounded predictions on the
sequence of examples. We show that a variant of the Online Newton
Step~\citep{HazanAgKa07}, while quadratic in computation, stays
regret-optimal with a nearly matching lower bound in this more general
setting.

\paragraph{2. Improved efficiency via sketching (Section~\ref{sec:sketch}).}
To overcome the quadratic running time, we next develop sketched variants of
the Newton update, approximating the second order information using a
small number of carefully chosen directions, called a \emph{sketch}.
While the idea of data sketching is widely studied~\citep{Woodruff14},
as far as we know our work is the first one to apply it to a general
adversarial online learning setting and provide rigorous regret
guarantees.  Two different sketching methods are considered: 
Frequent Directions~\citep{GhashamiLiPhWo15, Liberty13} and Oja's
algorithm~\citep{Oja82, OjaKa85}, both of which allow linear running
time per round. For the first method, we prove regret bounds
similar to the full second order update whenever the sketch-size is
large enough. 
Our analysis makes it easy to plug in other sketching and
online PCA methods (e.g.~\citep{garber2015online}).

\paragraph{3. Sparse updates (Section~\ref{sec:sparse}).}
For practical implementation, we further develop sparse versions of
these updates with a running time linear in the sparsity of the
examples.  The main challenge here is that even if examples are
sparse, the sketch matrix still quickly becomes dense.  These are the
first known sparse implementations of the Frequent
Directions\footnote{Recent work by~\citep{GhashamiLiPh16} also
  studies sparse updates for a more complicated variant of Frequent
  Directions which is randomized and incurs extra approximation
  error.}  and Oja's algorithm, and require new sparse eigen
computation routines that may be of independent interest.

Empirically, we evaluate our algorithm using the sparse Oja sketch
(called \ojaalg) against first order methods such as diagonalized
\textsc{AdaGrad}~\citep{DuchiHaSi2011, McMahanSt2010} on both
ill-conditioned synthetic and a suite of real-world datasets. As
Fig.~\ref{fig:synthetic} shows for a synthetic problem, we observe
substantial performance gains as data conditioning worsens. On the real-world datasets, we find improvements in some
instances, while observing no substantial second-order signal in the
others.

\paragraph{Related work}
Our online learning setting is closest to the one proposed
in~\citep{RossMiLa13}, which studies scale-invariant algorithms, a
special case of the invariance property considered here (see also
\citep[Section~5]{orabona2015}). Computational efficiency, a main
concern in this work, is not a problem there since each coordinate is
scaled independently. \citet{OrabonaPa15} study unrelated notions of
invariance.  \citet{GaoJiZhZh13} study a specific randomized sketching
method for a special online learning setting.

The L-BFGS algorithm~\citep{LiuNo89} has recently been studied in the
stochastic setting\footnote{Stochastic setting assumes that the
  examples are drawn i.i.d.\ from a
  distribution.}~\citep{ByrdHaNoSi14, MokhtariRi14, MoritzNiJo15,
  SchraudolphYuGu07, SohldicksteinPoGa14}, but has strong assumptions
with pessimistic rates in theory and reliance on the use of large
mini-batches empirically. Recent works \citep{ErdogduMo15,
  GonenOrSh16, GonenSh15, PilanciWa15} employ sketching in stochastic
optimization, but do not provide sparse implementations or extend in
an obvious manner to the online setting. The Frank-Wolfe
algorithm~\citep{FrankWo56, Jaggi13} is also invariant to linear
transformations, but with worse regret bounds~\citep{HazanKa12}
without further assumptions and modifications~\citep{GarberHa13}.


\paragraph{Notation}
Vectors are represented by bold letters (e.g., $\bx$, $\bw$, \dots) and
matrices by capital letters (e.g., $M$, $A$, \dots).  $M_{i,j}$ denotes the
$(i,j)$ entry of matrix $M$.  $\eye{d}$ represents the $d \times d$
identity matrix, $\bzero_{m\times d}$ represents the $m \times d$
matrix of zeroes, and $\diag{\bx}$ represents a diagonal matrix with
$\bx$ on the diagonal.  $\lambda_i(A)$ denotes the $i$-th largest
eigenvalue of $A$, $\norm{\bw}_A$ denotes $\sqrt{\bw^\top A \bw}$,
$|A|$ is the determinant of $A$, $\trace{A}$ is the trace of $A$,
$\inner{A, B}$ denotes $\sum_{i,j}A_{i,j}B_{i,j}$, and $A \preceq B$
means that $B - A$ is positive semidefinite. The sign function
$\sgn(a)$ is $1$ if $a\geq 0$ and $-1$ otherwise.

\section{Setup and an Optimal Algorithm}
\label{sec:setup}

We consider the following setting. On each round $t = 1,2\ldots, T$:
\textbf{(1)} the adversary first presents an example $\bx_t \in \R^d$,
\textbf{(2)} the learner chooses $\bw_t \in \R^d$ and predicts
$\bw_t^\top \bx_t$, \textbf{(3)} the adversary reveals a loss function
\mbox{$f_t(\bw) = \loss_t(\bw^\top \bx_t)$} for some convex,
differentiable $\loss_t: \R \rightarrow \R_+$, and \textbf{(4)} the
learner suffers loss $f_t(\bw_t)$ for this round.

The learner's regret to a comparator $\bw$ is defined as \mbox{$
  R_T(\bw) = \sum_{t=1}^T f_t(\bw_t) - \sum_{t=1}^T f_t(\bw)$}.
Typical results study $R_T(\bw)$ against all $\bw$ with a bounded norm
in some geometry. For an invariant update, we relax this requirement
and only put bounds on the predictions $\bw^\top \bx_t$. Specifically,
for some pre-chosen constant $C$ we define $ \scK_t \defeq
\theset{\bw}{|\bw^\top\bx_t| \leq C}.
$
We seek to minimize regret to all comparators that generate bounded
predictions on every data point, that is: 
$$R_T = \sup_{\bw \in \scK} R_T(\bw)~~\mbox{ where}~~ \scK \defeq
\bigcap_{t=1}^T \scK_t = \theset{\bw}{\forall t=1,2,\ldots
  T,~~|\bw^\top\bx_t| \leq C}~.$$   
Under this setup, if the data are transformed to $M\bx_t$ for all $t$
and some invertible matrix $M \in \R^{d\times d}$, the optimal $\bw^*$
simply moves to $(M^{-1})^\top \bw^*$, which still has bounded
predictions but might have significantly larger norm.  This relaxation
is similar to the comparator set considered in~\citep{RossMiLa13}.

We make two structural assumptions on the loss functions.
\begin{assumption}(Scalar Lipschitz)\label{ass:Lipschitz}
The loss function $\loss_t$ satisfies $|\loss_t^{'}(z)| \leq L$
whenever $|z| \leq C$.
\label{ass:loss}
\end{assumption}

\begin{assumption}(Curvature)\label{ass:curvature} There exists $\sigma_t \geq 0$ such that for all $\bu,
\bw \in \scK$, $f_t(\bw)$ is lower bounded by
$
 f_t(\bu) + \nabla f_t(\bu)^\top(\bw - \bu) + \frac{\sigma_t}{2}
 \left( \nabla f_t(\bu)^\top(\bu - \bw)\right)^2.
$
\label{ass:curve}
\end{assumption}
Note that when $\sigma_t = 0$, Assumption~\ref{ass:curve} merely imposes convexity. More
generally, it is satisfied by squared loss $f_t(\bw) = (\bw^\top\bx_t
- y_t)^2$ with $\sigma_t = \frac{1}{8C^2}$ whenever $|\bw^\top\bx_t|$
and $|y_t|$ are bounded by $C$, as well as for all exp-concave
functions (see~\citep[Lemma~3]{HazanAgKa07}).

Enlarging the comparator set might result in worse regret. We next
show matching upper and lower bounds qualitatively similar to the
standard setting, but with an extra unavoidable $\sqrt{d}$ factor.
\footnote{In the standard setting where $\bw_t$ and $\bx_t$ are restricted such that
$\norm{\bw_t} \leq D$ and $\norm{\bx_t} \leq X$, the minimax regret is
$\scO(DXL\sqrt{T})$. This is clearly a special case of our setting with $C = DX$.}

\begin{theorem}\label{thm:lower_bound}
For any online algorithm generating $\bw_t \in \R^d$ and all $T \geq d$,
there exists a sequence of $T$ examples $\bx_t \in \R^d$ and loss functions $\ell_t$ 
satisfying Assumptions~\ref{ass:loss} and~\ref{ass:curve} (with $\sigma_t = 0$) such that 
the regret $R_T$ is at least $CL\sqrt{dT/2}$.
\end{theorem}

We now give an algorithm that matches the lower bound up to
logarithmic constants in the worst case but enjoys much smaller regret
when $\sigma_t \neq 0$.  At round $t+1$ with some invertible matrix
$A_t$ specified later and gradient $\bg_t = \nabla f_t(\bw_t)$, the
algorithm performs the following update \emph{before} making the
prediction on the example $\bx_{t+1}$:
\begin{equation}\label{eq:AON}
\bu_{t+1} = \bw_t - A_t^{-1}\bg_t, \quad \mbox{and} \quad
\bw_{t+1} = \argmin_{\bw \in \scK_{t+1}} \norm{\bw-\bu_{t+1}}_{A_{t}}.
\end{equation}
The projection onto the set $\scK_{t+1}$ 
differs from typical norm-based projections as it only enforces
boundedness on $\bx_{t+1}$ at round $t+1$. 
Moreover, this projection step can be performed in closed form.
\begin{lemma}
For any $\bx \neq \bzero, \bu \in \R^{d}$ and positive definite matrix
$A \in \R^{d\times d}$, we have
\[ \argmin_{\bw \,:\, |\bw^\top\bx| \leq C} \norm{\bw-\bu}_{A} = 
 \bu - \frac{\tau_C(\bu^\top\bx)}{\bx^\top A^{-1} \bx} A^{-1}\bx,
 ~~\mbox{where $\tau_C(y) = \sgn(y)\max\{|y| - C, 0\}$.} 
\]
\label{lemma:projection}
\end{lemma}
%

If $A_t$ is a diagonal matrix, updates similar to those of~\citet{RossMiLa13} are recovered.
We study a choice of $A_t$ that is similar to the Online Newton Step (ONS)~\citep{HazanAgKa07}
(though with different projections):
\begin{equation}
A_t = \alpha \eye{d} + \sum_{s=1}^t (\sigma_s + \eta_s)\bg_s
\bg_s^\top
\label{eqn:ons-mat}
\end{equation}
for some parameters $\alpha > 0$ and $\eta_t \geq 0$. 
The regret guarantee of this algorithm is shown below:


\begin{theorem}\label{thm:AON}
Under Assumptions~\ref{ass:loss} and~\ref{ass:curvature}, suppose that $\sigma_t \geq \sigma
\geq 0$ for all $t$, and $\eta_t$ is non-increasing.  Then using the
matrices~\eqref{eqn:ons-mat} in the updates~\eqref{eq:AON} yields for
all $\bw\in\scK$,
\begin{align*}
R_T(\bw) \le \frac{\alpha}{2}\norm{\bw}_{2}^2 + 2(CL)^2
\sum_{t=1}^T\eta_t + \frac{d}{2(\sigma + \eta_T)} \ln\left(1 +
\frac{(\sigma+\eta_T)
  \sum_{t=1}^T\norm{\bg_t}_2^2}{d\alpha}\right)~.
\end{align*}
\end{theorem}
%
The dependence on $\norm{\bw}_2^2$ implies that the
method is not completely invariant to transformations of the data.
This is due to the part $\alpha\eye{d}$ in $A_t$.
However, this is not critical since $\alpha$ is fixed and small while the other
part of the bound grows to eventually become the dominating
term. Moreover, we can even set $\alpha = 0$ and replace the inverse with
the Moore-Penrose pseudoinverse to obtain a truly invariant algorithm, as discussed in
Appendix~\ref{app:pseudoinverse}.  We use $\alpha > 0$ in the
remainder for simplicity. 

The implication of this regret bound is the following:
in the worst case where $\sigma = 0$, we set $\eta_t = \sqrt{d/C^2L^2t}$ and the
bound simplifies to
\begin{align}
  R_T(\bw) \le \frac{\alpha}{2}\norm{\bw}_{2}^2 +
  \frac{CL}{2}\sqrt{Td} \ln\left(1 +
  \frac{\sum_{t=1}^T\norm{\bg_t}_2^2}{\alpha CL\sqrt{Td}}\right) \nonumber +
  4CL\sqrt{Td}~,
  \label{eq:aons-convex}
\end{align}
essentially only losing a logarithmic factor compared to the lower
bound in Theorem~\ref{thm:lower_bound}. 
On the other hand, if $\sigma_t \geq \sigma > 0$
for all $t$, then we set $\eta_t = 0$ and the regret simplifies to
\begin{equation}
  R_T(\bw) \le \frac{\alpha}{2}\norm{\bw}_{2}^2 +
  \frac{d}{2\sigma}\ln\left(1 + \frac{\sigma
    \sum_{t=1}^T\norm{\bg_t}_2^2}{d\alpha}\right)~,
  \label{eq:aons-sc}
\end{equation}
extending the $\scO(d\ln T)$ results in~\citep{HazanAgKa07} to the
weaker Assumption~\ref{ass:curve} and a larger comparator set $\scK$.

\section{Efficiency via Sketching}\label{sec:sketch}
Our algorithm so far requires $\Omega(d^2)$ time and space just as ONS. In this section we
show how to achieve regret guarantees nearly as good as the above
bounds, while keeping computation within a constant factor of
first order methods.

Let $G_t \in \R^{t\times d}$ be a matrix such that the $t$-th row is
$\bghat_t^\top$ where we define $\bghat_t = \sqrt{\sigma_t +
  \eta_t}\bg_t$ to be the \emph{to-sketch vector}. Our previous choice of
$A_t$ (Eq.~\eqref{eqn:ons-mat}) can be written as $\alpha\eye{d} + G_t^\top G_t$.
The idea of sketching is to maintain an approximation of $G_t$, denoted by $S_t \in
\R^{m\times d}$ where $m \ll d$ is a small constant called the sketch size. If $m$ is chosen
so that $S_t^\top S_t$ approximates $G_t^\top G_t$ well, we can redefine
$A_t$ as $\alpha\eye{d} + S_t^\top S_t$ for the algorithm.

\begin{algorithm}[t]
\caption{\alglong (\alg)}
\label{alg:SAON}
\begin{algorithmic}[1]
\Require Parameters $C$, $\alpha$ and $m$.
\State Initialize  $\bu_1 = \bzero_{d \times 1}$.
\State Initialize sketch $(S, H) \leftarrow \textbf{SketchInit}(\alpha, m)$.
\For{$t=1$ {\bfseries to} $T$}
    \State Receive example $\bx_t$.
    \State \textbf{Projection step:} compute $\bxhat = S\bx_t, \; \scale = \frac{\tau_C(\bu_t^\top\bx_t)}{\bx_{t}^\top \bx_{t} - {\bxhat}^\top H \bxhat}$ 
         and set $\bw_t = \bu_t -  \scale(\bx_t - S^\top H \bxhat)$.
    \State Predict label $y_t = \bw_t^\top \bx_t$ and suffer loss $\ell_t(y_t)$.
    \State Compute gradient $\bg_t = \ell'_t(y_t)\bx_t$ and the \emph{to-sketch vector} $\bghat = \sqrt{\sigma_t + \eta_t}\bg_t$.
    \State $(S, H) \leftarrow$ \textbf{SketchUpdate}($\bghat$). 
    \State\textbf{Update weight:} $\bu_{t+1} = \bw_t - \frac{1}{\alpha}(\bg_t - S^\top H S\bg_t)$.
\EndFor
\end{algorithmic}
\end{algorithm}

To see why this admits an efficient algorithm, notice that by the Woodbury formula one has
$
A_t^{-1} = \frac{1}{\alpha}\bigl(\eye{d} -
S_t^\top (\alpha\eye{m} + S_t S_t^\top)^{-1} S_t \bigr).
$ 
With the notation $H_t = (\alpha\eye{m} + S_t
S_t^\top)^{-1} \in \R^{m\times m}$ and $\scale_t =
\tau_C(\bu_{t+1}^\top\bx_{t+1})/(\bx_{t+1}^\top\bx_{t+1} -
\bx_{t+1}^\top S_t^\top H_t S_t \bx_{t+1})$,
update~(\ref{eq:AON}) becomes:
\begin{align*}
  \bu_{t+1} = \bw_t - \tfrac{1}{\alpha}\bigl(\bg_t - S_t^\top H_t S_t
  \bg_t\bigr), \quad \mbox{and} \quad \bw_{t+1} &= \bu_{t+1} - \scale_t
  \bigl(\bx_{t+1} - S_t^\top H_t S_t \bx_{t+1}\bigr)~. 
\end{align*}
The operations involving $S_t\bg_t$ or $S_t\bx_{t+1}$ require only
$\order(md)$ time, while matrix vector products with $H_t$ require
only $\order(m^2)$. Altogether, these updates are at most $m$ times
more expensive than first order algorithms as long as $S_t$ and $H_t$
can be maintained efficiently. We call this algorithm \alglong (\alg)
and summarize it in Algorithm~\ref{alg:SAON}.

We now discuss two sketching techniques to maintain the matrices
$S_t$ and $H_t$ efficiently, each requiring $\scO(md)$ storage
and time linear in $d$.

%

\paragraph{Frequent Directions (FD).}

\begin{figure*}[t]
  \centering
\begin{tabular}{@{}cc@{}}
\begin{minipage}{0.47\textwidth}
\begin{algorithm}[H]
\caption{FD-Sketch for \fdalg}
\label{alg:FD}
\begin{algorithmic}[1]
\Internal $S$ and $H$. 

\vspace{5pt}
\Init{$\alpha, m$}
\State Set $S = \bzero_{m\times d}$ and  $H = \tfrac{1}{\alpha}
\eye{m}$.
\State Return $(S, H)$.

\vspace{5pt}
\setcounter{ALG@line}{0}
\Update{$\bghat$}
\State Insert $\bghat$ into the last row of $S$.
\State Compute eigendecomposition: $V^\top \Sigma V = S^\top S$ and set $S = (\Sigma - \Sigma_{m,m}\eye{m})^{\frac{1}{2}} V$.  
\State Set $H = \diag{\frac{1}{\alpha + \Sigma_{1,1} - \Sigma_{m,m}},  \cdots, \frac{1}{\alpha}}$.
\State Return $(S, H)$.  
\end{algorithmic}
\end{algorithm}
\end{minipage} &

\begin{minipage}{0.47\textwidth}
  \begin{algorithm}[H]
\caption{Oja's Sketch for \ojaalg}
\label{alg:Oja}
\begin{algorithmic}[1]
\Internal $t$, $\Lambda$, $V$ and $H$. 

\vspace{5pt}
\Init{$\alpha, m$}
\State Set $t = 0, \Lambda = \bzero_{m \times m}, H =
\tfrac{1}{\alpha} \eye{m}$ and $V$ to any $m \times d$ matrix with orthonormal rows.
\State Return ($\bzero_{m \times d}$, $H$).

\vspace{6pt}
\setcounter{ALG@line}{0}
\Update{$\bghat$}
\State Update $t \leftarrow t + 1$, $\Lambda$ and $V$ as Eqn.~\ref{eqn:oja-eigs}.
\State Set $S = (t\Lambda)^{\frac{1}{2}} V$.
\State Set $H = \diag{\frac{1}{\alpha + t\Lambda_{1,1}},  \cdots, \frac{1}{\alpha + t\Lambda_{m,m}}}$.
\State Return $(S, H)$.
\end{algorithmic}
\end{algorithm}
\end{minipage}
\end{tabular}
\end{figure*}

  

%
Frequent Directions sketch~\citep{GhashamiLiPhWo15, Liberty13} is
a deterministic sketching method. 
It maintains the invariant that the last row of $S_t$ is always $\bzero$.  
On each round, the vector $\bghat_t^\top$ is inserted into the last row of $S_{t-1}$, then the
covariance of the resulting matrix is eigendecomposed into $V_t^\top
\Sigma_t V_t$ and $S_t$ is set to $(\Sigma_t -
\rho_{t}\eye{m})^{\frac{1}{2}} V_t$ where $\rho_t$ is the smallest
eigenvalue.  Since the rows of $S_t$ are orthogonal to each other,
$H_t$ is a diagonal matrix and can be maintained efficiently (see
Algorithm~\ref{alg:FD}).
The sketch update works in $\scO(md)$ time (see
\citep{GhashamiLiPhWo15} and Appendix~\ref{app:sparse}) so the
total running time is $\scO(md)$ per round.  
We call this combination \fdalg and prove the following regret bound
with notation $\Omega_k = \sum_{i=k+1}^d \lambda_i(G_T^\top G_T)$ for any $k = 0,\dots,m-1$.
\begin{theorem}\label{thm:FD}
Under Assumptions~\ref{ass:loss} and~\ref{ass:curvature}, suppose that $\sigma_t \geq \sigma \geq 0$ for all $t$ 
and $\eta_t$ is non-increasing. 
\fdalg ensures that for any $\bw \in \scK$ and $k = 0,\ldots,m-1$, we have
\begin{align*}
R_T(\bw) \le \frac{\alpha}{2}\norm{\bw}_{2}^2 + 2(CL)^2 \sum_{t=1}^T\eta_t + \frac{m}{2(\sigma +
  \eta_T)} \ln\left(1 + \frac{\trace{S_T^\top S_T}}{m\alpha}\right) +
\frac{m\Omega_k}{2(m-k)(\sigma+\eta_T)\alpha}~.
\end{align*}
\end{theorem}
The bound depends on the spectral decay
$\Omega_k$, which essentially is the only extra term compared to the
bound in Theorem~\ref{thm:AON}.  Similarly to previous discussion, if
$\sigma_t \geq \sigma$, we get the bound $ \frac{\alpha}{2}\norm{w}_{2}^2
+ \frac{m}{2\sigma}\ln\left(1+ \frac{\trace{S_T^\top
    S_T}}{m\alpha}\right) + \frac{m\Omega_k}{2(m-k)\sigma\alpha}~.  $
With $\alpha$ tuned well, we pay logarithmic regret for the top $m$
eigenvectors, but a square root regret $\scO(\sqrt{\Omega_k})$ for
remaining directions not controlled by our sketch. This is expected
for deterministic sketching which focuses on the dominant part of the
spectrum. When $\alpha$ is not tuned we still get sublinear regret as
long as $\Omega_k$ is sublinear. 


\paragraph{Oja's Algorithm.}

Oja's algorithm \citep{Oja82, OjaKa85} is not usually considered as a
sketching algorithm but seems very natural here.  This algorithm uses
online gradient descent to find eigenvectors and eigenvalues of
data in a streaming fashion, with the to-sketch vector $\bghat_t$'s as
the input.  Specifically, let $V_t \in \R^{m \times d}$ denote the
estimated eigenvectors and the diagonal matrix $\Lambda_t \in \R^{m
  \times m}$ contain the estimated eigenvalues at the end of round
$t$.  Oja's algorithm updates as:
\begin{align}
\Lambda_{t} = (\eye{m} - \Gamma_t) \Lambda_{t-1} + \Gamma_t
\;\diag{V_{t-1} \bghat_t}^2, \quad\quad
V_t \xleftarrow{\text{orth}} V_{t-1} + \Gamma_t V_{t-1}\bghat_t
\bghat_t^\top
\label{eqn:oja-eigs}
\end{align}
where $\Gamma_t \in \R^{m \times m}$ is a diagonal matrix with
(possibly different) learning rates of order $\Theta(1/t)$ on the
diagonal, and the ``$\xleftarrow{\text{orth}}$'' operator represents
an orthonormalizing step.\footnote{For simplicity, we assume that
  $V_{t-1} + \Gamma_t V_{t-1}\bghat_t \bghat_t^\top$ is always of full
  rank so that the orthonormalizing step does not reduce the dimension
  of $V_t$. \label{fn:full_rank}} The sketch is then $S_t = (t
\Lambda_t)^{\frac{1}{2}} V_t$.  The rows of $S_t$ are orthogonal and
thus $H_t$ is an efficiently maintainable diagonal matrix (see
Algorithm~\ref{alg:Oja}). We call this combination \ojaalg.

The time complexity of Oja's algorithm is $\scO(m^2d)$ per round due to the
orthonormalizing step.  To improve the running time to $\scO(md)$, one
can only update the sketch every $m$ rounds (similar to the block
power method~\citep{HardtPr14, LiLiLu15}).  The regret guarantee of
this algorithm is unclear since existing analysis for Oja's algorithm
is only for the stochastic setting (see
e.g.~\citep{BalsubramaniDaFr13, LiLiLu15}).  However, \ojaalg provides
good performance experimentally.

\section{Sparse Implementation}\label{sec:sparse}
In many applications, examples (and hence gradients) are sparse in the
sense that $\norm{\bx_t}_{0} \leq s$ for all $t$ and some small
constant $s \ll d$.  Most online first order methods enjoy a
per-example running time depending on $s$ instead of $d$ in such
settings.  Achieving the same for second order methods is more
difficult since $A_t^{-1}\bg_t$ (or sketched versions) are typically
dense even if $\bg_t$ is sparse.

We show how to implement our algorithms in sparsity-dependent time,
specifically, in $\scO(m^2 + ms)$ for \fdalg and in
$\scO(m^3 + m s)$ for \ojaalg.  We emphasize that since the sketch
would still quickly become a dense matrix even if the examples are
sparse, achieving purely sparsity-dependent time is highly non-trivial 
and may be of independent interest.  
Due to space limit, below we only briefly mention how to do
it for \ojaalg.  Similar discussion for the FD sketch can be
found in Appendix~\ref{app:sparse}.  Note that mathematically these
updates are equivalent to the non-sparse counterparts and regret
guarantees are thus unchanged.

There are two ingredients to doing this for \ojaalg:
(1) The eigenvectors $V_t$ are represented as $V_t = F_t Z_t$, where
$Z_t \in \R^{m\times d}$ is a sparsely updatable direction (Step 3 in
Algorithm~\ref{alg:SOja}) and $F_t \in \R^{m\times m}$ is a matrix such
that $F_t Z_t$ is orthonormal.  (2) The weights $\bw_t$ are split as
$\bbw_t + Z_{t-1}^\top \bb_t$, where $\bb_t \in \R^m$ maintains the
weights on the subspace captured by $V_{t-1}$ (same as
$Z_{t-1}$), and $\bbw_t$ captures the weights on the complementary
subspace which are again updated sparsely.

We describe the sparse updates for $\bbw_t$ and $\bb_t$ below with the
details for $F_t$ and $Z_t$ deferred to
Appendix~\ref{app:sparse-oja}. Since $S_t = (t
\Lambda_t)^{\frac{1}{2}} V_t = (t \Lambda_t)^{\frac{1}{2}} F_tZ_t$ and
$\bw_t = \bbw_t + Z_{t-1}^\top\bb_t$, we know $\bu_{t+1}$ is 
\begin{align}
\bw_t - \big(\eye{d} - S_t^\top H_t
S_t\big)\tfrac{\bg_t}{\alpha} = \underbrace{\bbw_t - \tfrac{\bg_t}{\alpha} - (Z_t -
  Z_{t-1})^\top \bb_t}_{\defeq \bbu_{t+1}} + Z_t^\top
(\underbrace{\bb_t + \tfrac{1}{\alpha} F_t^\top (t\Lambda_t H_t) F_t
  Z_t \bg_t}_{\defeq \bb_{t+1}'})~.
\label{eqn:u-oja}
\end{align}
Since $Z_t - Z_{t-1}$ is sparse by construction and the
matrix operations defining $\bb_{t+1}'$ scale with $m$, overall the
update can be done in $\scO(m^2 + ms)$.  Using the update for
$\bw_{t+1}$ in terms of $\bu_{t+1}$, $\bw_{t+1}$ is equal to
\begin{align}
  \bu_{t+1} - \scale_t (\eye{d} - S_t^\top H_t S_t ) \bx_{t+1} =
  \underbrace{\bbu_{t+1} - \scale_{t}\bx_{t+1}}_{\defeq \bbw_{t+1}} +
  Z_t^\top (\underbrace{\bb_{t+1}' + \scale_{t} F_t^\top (t\Lambda_t
    H_t) F_t Z_t\bx_{t+1}}_{\defeq \bb_{t+1}})~.
  \label{eqn:w-oja}
\end{align}
%
Again, it is clear that all the computations scale with $s$ and not $d$, so both
$\bbw_{t+1}$ and $\bb_{t+1}$ require only $O(m^2+ms)$ time to
maintain. Furthermore, the prediction $\bw_t^\top \bx_t = \bbw_t^\top
\bx_t + \bb_t^\top Z_{t-1} \bx_t$ can also be computed in $\scO(ms)$
time. The $\scO(m^3)$ in the overall complexity comes from a
Gram-Schmidt step in maintaining $F_t$ (details in
Appendix~\ref{app:sparse-oja}).

The pseudocode is presented in Algorithms~\ref{alg:SON}
and~\ref{alg:SOja} with some details deferred to
Appendix~\ref{app:sparse-oja}. This is the first sparse implementation
of online eigenvector computation to the best of our knowledge.

\begin{algorithm}[t]
\caption{Sparse \alglong with Oja's Algorithm}
\label{alg:SON}
\begin{algorithmic}[1]
\Require Parameters $C$, $\alpha$ and $m$.
\State Initialize $\bbu = \bzero_{d \times 1}$ and $\bb = \bzero_{m \times 1}$. 
\State ($\Lambda, F, Z, H) \leftarrow \textbf{SketchInit}(\alpha, m)$ \quad (Algorithm~\ref{alg:SOja}).
\For{$t=1$ {\bfseries to} $T$}
    \State Receive example $\bx_t$.
    \State \textbf{Projection step:} compute $\bxhat = FZ\bx_t$ and
    $\scale = \frac{\tau_C(\bbu^\top\bx_t + \bb^\top
      Z\bx_t)}{\bx_{t}^\top \bx_{t} - (t-1){\bxhat}^\top \Lambda H
      \bxhat}$. 
      
    Obtain $\bbw = \bbu -  \scale \bx_t$ and $\bb \leftarrow
    \bb + \scale(t-1)F^\top \Lambda H \bxhat $ \quad (Equation~\ref{eqn:w-oja}).
    \State Predict label $y_t = \bbw^\top \bx_t + \bb^\top Z\bx_t$ and suffer loss $\ell_t(y_t)$.
    \State Compute gradient $\bg_t = \ell'_t(y_t)\bx_t$ and the \emph{to-sketch vector} $\bghat = \sqrt{\sigma_t + \eta_t}\bg_t$.
    \State ($\Lambda$, $F$, $Z$, $H$, $\bdelta$) $\leftarrow$
    \textbf{SketchUpdate}($\bghat$) \quad (Algorithm~\ref{alg:SOja}).
    \State \textbf{Update weight:} $\bbu = \bbw - \tfrac{1}{\alpha}
    \bg_t - (\bdelta^\top\bb) \bghat  $ and $\bb \leftarrow \bb +
    \tfrac{1}{\alpha}tF^\top \Lambda HF Z\bg_t$ \quad (Equation~\ref{eqn:u-oja}).
    \EndFor
\end{algorithmic}
\end{algorithm}

\begin{algorithm}[t]
\caption{Sparse Oja's Sketch}
\label{alg:SOja}
\begin{algorithmic}[1]
\Internal $t$, $\Lambda$, $F$, $Z$, $H$ and $K$. 

\vspace{5pt}
\Init{$\alpha, m$}
\State Set $t = 0, \Lambda = \bzero_{m \times m}, F = K = \alpha H =
\eye{m}$ and $Z$ to any $m \times d$ matrix with 
orthonormal rows.
\State Return ($\Lambda$, $F$, $Z$, $H$).

\vspace{5pt}
\setcounter{ALG@line}{0}
\Update{$\bghat$}
    \State Update $t \leftarrow t + 1$. Pick a diagonal stepsize
    matrix $\Gamma_t$ to update $\Lambda \leftarrow (\eye{} -
    \Gamma_t) \Lambda + \Gamma_t \;\diag{FZ \bghat}^2$. 
    \State Set $\bdelta = A^{-1}\Gamma_t FZ \bghat$ and update $K
    \leftarrow K + \bdelta \bghat^\top Z^\top + Z \bghat \bdelta^\top
    + (\bghat^\top \bghat) \bdelta \bdelta^\top $. 
    \State Update $Z \leftarrow Z +  \bdelta \bghat^\top$. 
    \State $(L, Q) \leftarrow \text{Decompose}(F, K)$
    (Algorithm~\ref{alg:Gram-Schmidt}), so that $LQZ = FZ$ and $QZ$ is
    orthogonal. Set $F = Q$.
    \State Set $H \leftarrow \diag{\frac{1}{\alpha + t \Lambda_{1,1}},  \cdots, \frac{1}{\alpha + t \Lambda_{m,m}}}$.
    \State Return ($\Lambda$, $F$, $Z$, $H$, $\bdelta$).
\end{algorithmic}
\end{algorithm}

\section{Experiments}
Preliminary experiments revealed that out of our two sketching
options, Oja's sketch generally has better performance (see
Appendix~\ref{app:experiment}). For more thorough evaluation, we
implemented the sparse version of \ojaalg in Vowpal
Wabbit.\footnote{An open source machine learning toolkit available at
  \url{http://hunch.net/\~ vw}} We compare it with
\textsc{AdaGrad}~\citep{DuchiHaSi2011, McMahanSt2010} on both
synthetic and real-world datasets. Each algorithm takes a stepsize
parameter: $\tfrac{1}{\alpha}$ serves as a stepsize for \ojaalg and a
scaling constant on the gradient matrix for \adagrad. We try both
methods with the parameter set to $2^j$ for $j = -3, -2, \ldots,
6$ and report the best results. We keep the stepsize matrix in \ojaalg
fixed as $\Gamma_t = \frac{1}{t}\eye{m}$ throughout. All methods make
one online pass over data minimizing square loss.

\subsection{Synthetic Datasets}

\begin{figure}[t]
\begin{minipage}{.32\textwidth}
\begin{figure}[H]
 \includegraphics[width=1\textwidth]{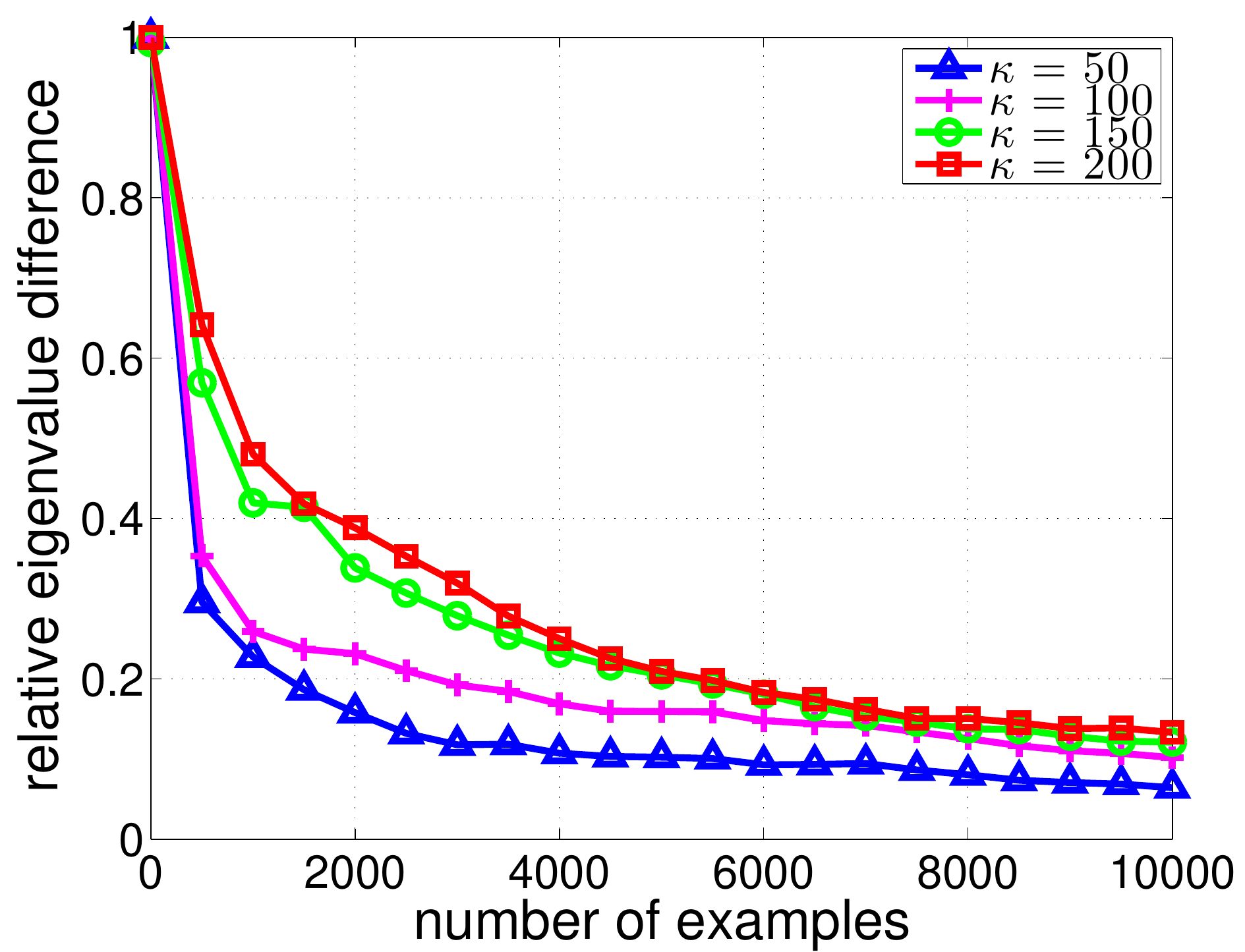}
\caption{Oja's algorithm's eigenvalue recovery error.}
\label{fig:eigs}
\end{figure}
\end{minipage}
\begin{minipage}{.08\textwidth}
\quad
\end{minipage}
\begin{minipage}{.6\textwidth}
\begin{figure}[H]
  \centering
  \subfigure {\includegraphics[width=.48\textwidth]{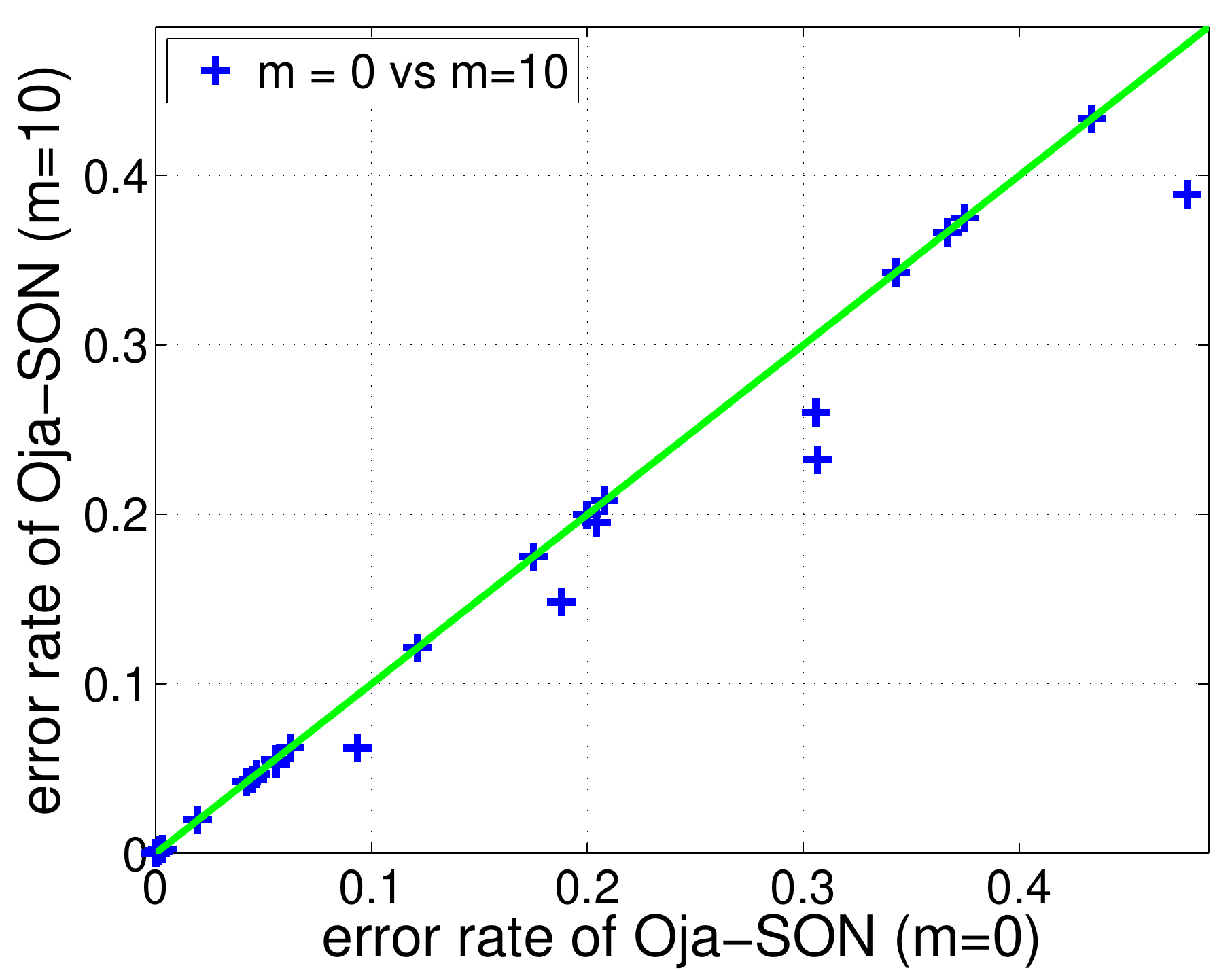}
    \label{fig:nodiag}
  }\hfill 
  \subfigure
      {\includegraphics[width=.49\textwidth]{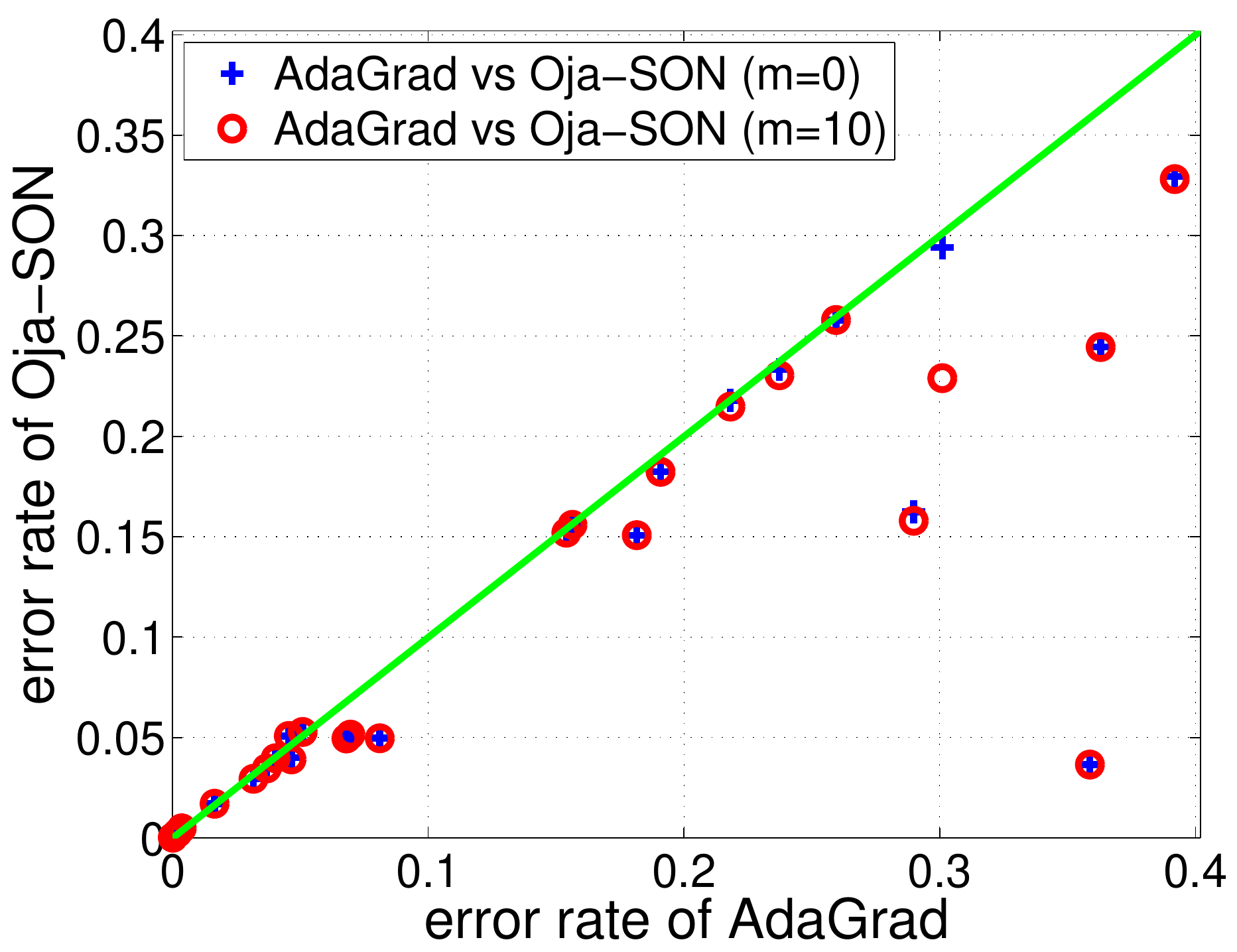} \label{fig:diag}}
\caption{(a) Comparison of two sketch sizes on real data, and (b)
Comparison against \adagrad on real data.}
\end{figure}
\end{minipage}
\end{figure}

To investigate \ojaalg's performance in the setting it is really
designed for, we generated a range of synthetic ill-conditioned
datasets as follows. We picked a random Gaussian matrix $Z\sim
\R^{T\times d}$ ($T = 10,\!000$ and $d = 100$) and a random
orthonormal basis $V \in \R^{d\times d}$. We chose a specific spectrum
$\blambda \in \R^d$ where the first $d-10$ coordinates are 1 and the
rest increase linearly to some fixed \emph{condition number parameter
  $\kappa$}. We let $X = Z\diag{\blambda}^{\frac{1}{2}} V^\top$ be our
example matrix, and created a binary classification problem with
labels $y = \sign(\btheta^\top\bx)$, where $\btheta \in \R^d$ is a
random vector.  We generated 20 such datasets with the same $Z, V$ and
labels $y$ but different values of $\kappa \in \{10, 20, \ldots,
200$\}.  Note that if the algorithm is truly invariant, it would have
the same behavior on these 20 datasets.

Fig.~\ref{fig:synthetic} (in Section~\ref{sec:intro}) shows the final
progressive error (i.e. fraction of misclassified examples after one pass over data) for
\textsc{AdaGrad} and \ojaalg (with sketch size $m=0,5,10$) as the
condition number increases.  As expected, the plot confirms the
performance of first order methods such as \textsc{AdaGrad} degrades
when the data is ill-conditioned.  The plot also shows that as
the sketch size increases, \ojaalg becomes more accurate: when $m=0$
(no sketch at all), \ojaalg is vanilla gradient descent and is worse
than \textsc{AdaGrad} as expected; when $m=5$, the accuracy greatly
improves; and finally when $m=10$, the accuracy of \ojaalg is
substantially better and hardly worsens with $\kappa$.

To further explain the effectiveness of Oja's algorithm in identifying
top eigenvalues and eigenvectors, the plot in Fig.~\ref{fig:eigs}
shows the largest relative difference between the true and estimated
top 10 eigenvalues as Oja's algorithm sees more data.  This gap drops
quickly after seeing just 500 examples.

\subsection{Real-world Datasets}\label{subsec:real_data}

Next we evaluated \ojaalg on 23 benchmark datasets from the UCI and
LIBSVM repository (see Appendix~\ref{app:experiment} for description
of these datasets).  Note that some datasets are very high dimensional
but very sparse (e.g. for {\it 20news}, $d \approx 102,000$ and $s
\approx 94$), and consequently methods with running time quadratic
(such as ONS) or even linear in dimension rather than sparsity are
prohibitive.

In Fig.~\ref{fig:nodiag}, we show the effect of using sketched second order information, 
by comparing sketch size $m=0$ and $m=10$ for \ojaalg (concrete error rates in
Appendix~\ref{app:experiment}).  We observe significant improvements
in 5 datasets ({\it acoustic, census, heart, ionosphere, letter}),
demonstrating the advantage of using second order information.
However, we found that \ojaalg was outperformed by \textsc{AdaGrad} on
most datasets, mostly because the diagonal adaptation of
\textsc{AdaGrad} greatly reduces the condition number on these
datasets.  Moreover, one disadvantage of \alg is that for the
directions not in the sketch, it is essentially doing vanilla gradient
descent.  We expect better results using diagonal adaptation as in
\textsc{AdaGrad} in off-sketch directions.

To incorporate this high level idea, we performed a simple
modification to \ojaalg: upon seeing example $\bx_t$, we feed
$D_t^{-\frac{1}{2}} \bx_t$ to our algorithm instead of $\bx_t$, where
$D_t \in \R^{d\times d}$ is the diagonal part of the matrix
$\sum_{\tau=1}^{t-1} \bg_\tau\bg_\tau^\top$.\footnote{$D_1$ is defined
  as $0.1\times \eye{d}$ to avoid division by zero.} The intuition is
that this diagonal rescaling first homogenizes the scales of all
dimensions. Any remaining ill-conditioning is further addressed by the
sketching to some degree, while the complementary subspace is no
worse-off than with \textsc{AdaGrad}. We believe this flexibility in
picking the right vectors to sketch is an attractive aspect of our
sketching-based approach.

With this modification, \ojaalg outperforms \textsc{AdaGrad} on most
of the datasets even for $m = 0$, as shown in Fig.~\ref{fig:diag}
(concrete error rates in Appendix~\ref{app:experiment}). The
improvement on \textsc{AdaGrad} at $m=0$ is surprising but not
impossible as the updates are not identical--our update is scale
invariant like~\citet{RossMiLa13}. However, the diagonal adaptation
already greatly reduces the condition number on all datasets except
{\it splice} (see Fig.~\ref{fig:splice} in
Appendix~\ref{app:experiment} for detailed results on this dataset),
so little improvement is seen for sketch size $m=10$ over $m=0$.  For
several datasets, we verified the accuracy of Oja's method in
computing the top-few eigenvalues (Appendix~\ref{app:experiment}), so
the lack of difference between sketch sizes is due to the lack of
second order information after the diagonal correction.

The average running time of our algorithm when $m=10$ is about 11
times slower than \textsc{AdaGrad}, matching expectations. Overall,
\alg can significantly outperform baselines on ill-conditioned data,
while maintaining a practical computational complexity.

\paragraph{Acknowledgements} This work was done when Haipeng Luo and Nicol\`o Cesa-Bianchi were at Microsoft Research, New York.
We thank Lijun Zhang for pointing out our mistake in the regret proof
of another sketching method that appeared in an earlier version.



\newpage
\bibliographystyle{abbrvnat}
{\small \bibliography{../nips/ref_short}}

\newpage
\appendix

\begin{center}
\bf \Large Supplementary material for \\ ``Efficient Second Order Online Learning by Sketching''
\end{center}

\section{Proof of Theorem~\ref{thm:lower_bound}}\label{app:lower_bound}
\begin{proof}
Assuming $T$ is a multiple of $d$ without loss of generality, we pick $\bx_t$ from the basis vectors $\{\be_1, \ldots, \be_d\}$ 
so that each $\be_i$ appears $T/d$ times (in an arbitrary order). Note that now $\scK$ is just a hypercube:
\[
  \scK = \theset{\bw}{|\bw^\top\bx_t| \leq C, \;\;\forall t} = \theset{\bw}{\norm{\bw}_\infty \leq C}.
\]

Let $\xi_1,\dots,\xi_T$ be independent Rademacher random
variables such that $\Pr(\xi_t = +1) = \Pr(\xi_t = -1) =
\tfrac{1}{2}$.  For a scalar $\theta$, we define loss
function\footnote{By adding a suitable constant, these losses can
  always be made nonnegative while leaving the regret unchanged.}
$\ell_t(\theta) = (\xi_t L)\theta$, so that
Assumptions~\ref{ass:loss} and~\ref{ass:curve} are clearly satisfied
with $\sigma_t = 0$.  We show that, for any online algorithm,
\[
    \E[ R_T] = \E\left[\sum_{t=1}^T \ell_t\bigl(\bw_t^{\top}\bx_t\bigr) - \inf_{\bw\in\scK} \sum_{t=1}^T \ell_t\bigl(\bw^{\top}\bx_t\bigr) \right] \ge CL\sqrt{\frac{dT}{2}}
\]
which implies the statement of the theorem.

First of all, note that $\E\Bigl[\loss_t\bigl(\bw_t^{\top}\bx_t\bigr)
  \,\Big|\, \xi_1,\dots,\xi_{t-1} \Bigr] = 0$ for any
$\bw_t$. Hence we have
\begin{align*}
    \E\left[\sum_{t=1}^T \loss_t\bigl(\bw_t^{\top}\bx_t\bigr) - \inf_{\bw\in\scK} \sum_{t=1}^T \loss_t\bigl(\bw^{\top}\bx_t\bigr) \right]
&=
    \E\left[\sup_{\bw\in\scK} \sum_{t=1}^T -\loss_t\bigl(\bw^{\top}\bx_t\bigr) \right]
=
    L\,\E\left[\sup_{\bw\in\scK} \bw^{\top}\sum_{t=1}^T \xi_t \bx_t \right],
\end{align*}
which, by the construction of $\bx_t$, is 
\[
  CL\,\E\left[  \norm{\sum_{t=1}^T \xi_t \bx_t}_1 \right] = CLd\,\E\left[ \left|\sum_{t=1}^{T/d} \xi_t \right| \right] 
  \geq CLd \sqrt{\frac{T}{2d}} = CL\sqrt{\frac{dT}{2}},
\]
where the final bound is due to the Khintchine inequality (see
e.g. Lemma 8.2 in~\cite{CesabianchiLu06}).  This concludes the proof.
\end{proof}

\section{Projection}\label{app:projection}

We prove a more general version of Lemma~\ref{lemma:projection} which
does not require invertibility of the matrix $A$ here.

\begin{lemma}  
For any $\bx \neq \bzero, \bu \in \R^{d\times 1}$ and positive semidefinite matrix $A \in \R^{d\times d}$, we have
\[ \bw^* = \argmin_{\bw: |\bw^\top\bx| \leq C} \norm{\bw-\bu}_{A} = 
\left\{ \begin{array}{cl}
        \bu - \frac{\tau_C(\bu^\top\bx)}{\bx^\top A^{\dagger} \bx} A^{\dagger}\bx & \text{if $\bx \in \range(A)$}
    \\
    \\
        \bu - \frac{\tau_C(\bu^\top\bx)}{\bx^\top (\eye{} - A^{\dagger}A) \bx} (\eye{} - A^{\dagger}A) \bx & \text{if $\bx \notin \range(A)$}
    \end{array} \right.
\]
where $\tau_C(y) = \sgn(y)\max\{|y| - C, 0\}$ and $A^{\dagger}$ is the Moore-Penrose pseudoinverse of $A$. (Note that when $A$ is rank deficient, this is one of the many possible solutions.)
\end{lemma}
\begin{proof}
First consider the case when $\bx \in \range(A)$.
If $|\bu^\top\bx| \leq C$, then it is trivial that $\bw^* = \bu$.
We thus assume $\bu^\top\bx \geq C$ below (the last case $\bu^\top\bx \leq -C$ is similar). 
The Lagrangian of the problem is 
\[ L(\bw, \lambda_1, \lambda_2) = \frac{1}{2}(\bw-\bu)^\top A(\bw-\bu) + \lambda_1(\bw^\top\bx - C) + \lambda_2(\bw^\top\bx + C)\]
where $\lambda_1 \geq 0$ and $\lambda_2 \leq 0$ are Lagrangian multipliers. 
Since $\bw^\top\bx$ cannot be $C$ and $-C$ at the same time,
The complementary slackness condition implies that either $\lambda_1 = 0$ or $\lambda_2 = 0$.
Suppose the latter case is true, then setting the derivative with respect to $\bw$ to $0$, we get $\bw^* = \bu - \lambda_1 A^{\dagger}\bx + (\eye{} - A^\dagger A)\bz$ where $\bz \in R^{d \times 1}$ can be arbitrary.
However, since $A (\eye{} - A^\dagger A) = 0$, this part does not affect the objective value at all
and we can simply pick $z = 0$ so that $w^*$ has a consistent form regardless of whether $A$ is full rank or not.
Now plugging $\bw^*$ back, we have
\[ L(\bw^*, \lambda_1, 0) = -\frac{{\lambda_1}^2}{2}\bx^\top A^{\dagger} \bx + \lambda_1 (\bu^\top\bx - C) \]
which is maximized when $\lambda_1 = \frac{\bu^\top\bx - C}{\bx^\top A^{\dagger} \bx} \geq 0$. 
Plugging this optimal $\lambda_1$ into $\bw^*$ gives the stated solution.
On the other hand, if $\lambda_1 = 0$ instead, we can proceed similarly and verify that it gives a smaller dual value ($0$ in fact),
proving the previous solution is indeed optimal.

We now move on to the case when $\bx \notin \range(A)$. 
First of all the stated solution is well defined since $\bx^\top (\eye{} - A^{\dagger}A) \bx$ is nonzero in this case.
Moreover, direct calculation shows that $\bw^*$ is in the valid space: $|{\bw^*}^\top\bx| = |\bu^\top\bx - \tau_C(\bu^\top\bx)| \leq C$,
and also it gives the minimal possible distance value $\norm{\bw^*-\bu}_{A} = 0$,
proving the lemma.
\end{proof}


\section{Proof of Theorem~\ref{thm:AON}}\label{app:AON}

We first prove a general regret bound that holds for any choice of $A_t$ in update~\ref{eq:AON}:
\begin{equation*}
\begin{split}
\bu_{t+1} &= \bw_t - A_t^{-1}\bg_t \\
\bw_{t+1} &= \argmin_{\bw \in \scK_{t+1}} \norm{\bw-\bu_{t+1}}_{A_{t}}~.
\end{split}
\end{equation*}
This bound will also be useful in proving regret guarantees for the sketched versions.

\begin{prop}\label{prop:meta}
  For any sequence of positive definite matrices $A_t$ and sequence of losses satisfying
  Assumptions~\ref{ass:loss} and~\ref{ass:curvature}, the regret of
  updates~\eqref{eq:AON} against any comparator $\bw \in \scK$
  satisfies
  \begin{equation}
    2R_T(\bw) \le \|\bw\|_{A_0}^2 + \underbrace{\sum_{t=1}^T
      \bg_t^TA_t^{-1}\bg_t}_{\text{``Gradient Bound'' $R_G$}} \nonumber +
      \underbrace{\sum_{t=1}^T (\bw_t - \bw)^\top(A_t - A_{t-1} - \sigma_t \bg_t
    \bg_t^\top)(\bw_t - \bw)}_{\text{``Diameter Bound'' $R_D$}}
    \label{eq:regret-meta}
  \end{equation}
\end{prop}

\begin{proof}
  Since $\bw_{t+1}$ is the projection of $\bu_{t+1}$ onto $\scK_{t+1}$,
  by the property of projections (see for example~\citep[Lemma 8]{HazanKa12}), the algorithm ensures
  \[
  \norm{\bw_{t+1}-\bw}_{A_t}^2 \leq \norm{\bu_{t+1}-\bw}_{A_t}^2 
  =
  \norm{\bw_{t}-\bw}_{A_t}^2 + \bg_t^\top A_t^{-1} \bg_t - 2\bg_t^\top(\bw_t - \bw)
  \]
  for all $\bw\in\scK\ss\scK_{t+1}$. By
  the curvature property in Assumption~\ref{ass:curvature}, we then have that 
  \begin{align*}
    2R_T(\bw) \;&\leq\; \sum_{t=1}^T 2\bg_t^\top(\bw_t - \bw) -
    \sigma_t\bigl(\bg_t^\top(\bw_t - \bw)\bigr)^2 \\ 
    \;&\leq\; \sum_{t=1}^T \bg_t^\top A_t^{-1} \bg_t +
    \norm{\bw_{t}-\bw}_{A_t}^2 - \norm{\bw_{t+1}-\bw}_{A_t}^2 -
    \sigma_t\bigl(\bg_t^\top(\bw_t - \bw)\bigr)^2 \\ 
    \;&\leq\; \norm{\bw}_{A_0}^2 + \sum_{t=1}^T \bg_t^\top A_t^{-1} \bg_t  + 
  (\bw_t - \bw)^\top(A_t - A_{t-1} - \sigma_t \bg_t\bg_t^\top)(\bw_t -
    \bw),
\end{align*}
which completes the proof.
\end{proof}

\begin{proof}[Proof of Theorem~\ref{thm:AON}]

We apply Proposition~\ref{prop:meta} with the
choice: $A_0 = \alpha\eye{d}$ and $A_t = A_{t-1} + (\sigma_t +
\eta_t)\bg_t\bg_t^T$, which gives $\norm{\bw}_{A_0}^2 = \alpha\norm{\bw}_{2}^2$ and
\[
  R_D = \sum_{t=1}^T \eta_t (\bw_t - \bw)^\top  \bg_t\bg_t^\top (\bw_t - \bw) \leq 4(CL)^2 \sum_{t=1}^T \eta_t~,
\]
where the last equality uses the Lipschitz property in
Assumption~\ref{ass:loss} and the boundedness of $\bw_t^\top \bx_t$ and $\bw^\top \bx_t$.

For the term $R_G$, define $\Ahat_t = \frac{\alpha}{\sigma+\eta_T}\eye{d} + \sum_{s=1}^t
\bg_s \bg_s^\top$.
Since $\sigma_t \geq \sigma$ and $\eta_t$ is non-increasing, we have
$ \Ahat_t \preceq \frac{1}{\sigma+\eta_T} A_t$, and therefore:
\begin{align*}
R_G &\leq \frac{1}{\sigma+\eta_T}
\sum_{t=1}^T \bg_t^\top \Ahat_t^{-1} \bg_t = \frac{1}{\sigma+\eta_T}
\sum_{t=1}^T \inner{\Ahat_t - \Ahat_{t-1},\; \Ahat_t^{-1} } \\
&\leq \frac{1}{\sigma+\eta_T}  \sum_{t=1}^T
\ln\frac{|\Ahat_t|}{|\Ahat_{t-1}|} 
= \frac{1}{\sigma+\eta_T}  \ln\frac{|\Ahat_T|}{|\Ahat_{0}|} \\
&= \frac{1}{\sigma+\eta_T} \sum_{i=1}^d \ln\left(1 +
\frac{(\sigma+\eta_T)\lambda_i\Bigl(\sum_{t=1}^T \bg_t
  \bg_t^\top\Bigr)}{\alpha} \right) \\ 
&\leq \frac{d}{\sigma+\eta_T} \ln\left(1 +  \frac{(\sigma+\eta_T)
  \sum_{i=1}^d \lambda_i\Bigl(\sum_{t=1}^T \bg_t
  \bg_t^\top\Bigr)}{d\alpha} \right) \\ 
&= \frac{d}{\sigma + \eta_T} \ln\left(1 + \frac{(\sigma+\eta_T)
  \sum_{t=1}^T\norm{\bg_t}_2^2}{d\alpha}\right) 
\end{align*}
where the second inequality is by the concavity of the function
$\ln|X|$ (see~\citep[Lemma 12]{HazanAgKa07} for an alternative proof),
and the last one is by Jensen's inequality. 
This concludes the proof.
\end{proof}

\section{A Truly Invariant Algorithm}\label{app:pseudoinverse}
In this section we discuss how to make our adaptive online Newton algorithm truly invariant to invertible linear transformations.
To achieve this, we set $\alpha = 0$ and replace $A_t^{-1}$ with the Moore-Penrose pseudoinverse $A_t^{\dagger}$: 
\footnote{See Appendix~\ref{app:projection} for the closed form of the projection step.}
\begin{equation}\label{eq:pseudo_AON}
\begin{split}
\bu_{t+1} &= \bw_t - A_t^{\dagger}\bg_t, \\
\bw_{t+1} &= \argmin_{\bw \in \scK_{t+1}} \norm{\bw-\bu_{t+1}}_{A_t}~.
\end{split}
\end{equation}
When written in this form, it is not immediately clear that the algorithm has the invariant property. 
However, one can rewrite the algorithm in a mirror descent form:
\begin{align*}
\bw_{t+1} &= \argmin_{\bw \in \scK_{t+1}} \norm{\bw-\bw_t + A_t^{\dagger}\bg_t}_{A_t}^2 \\
&= \argmin_{\bw \in \scK_{t+1}} \norm{\bw-\bw_t}_{A_t}^2 + 2(\bw-\bw_t)^\top A_t A_t^{\dagger} \bg_t \\
&= \argmin_{\bw \in \scK_{t+1}} \norm{\bw-\bw_t}_{A_t}^2 + 2\bw^\top \bg_t
\end{align*}
where we use the fact that $\bg_t$ is in the range of $A_t$ in the last step.
Now suppose all the data $\bx_t$ are transformed to $M\bx_t$ for some unknown and invertible matrix $M$,
then one can verify that all the weights will be transformed to $M^{-T}\bw_t$ accordingly, ensuring the prediction to remain the same.

Moreover, the regret bound of this algorithm can be bounded as below. 
First notice that even when $A_t$ is rank deficient, the projection step still ensures the following:
$\norm{\bw_{t+1}-\bw}_{A_t}^2 \leq \norm{\bu_{t+1}-\bw}_{A_t}^2 $,
which is proven in \citep[Lemma~8]{HazanAgKa07}.
Therefore, the entire proof of Theorem~\ref{thm:AON} still holds after replacing $A_t^{-1}$ with $A_t^{\dagger}$, giving the regret bound:
\begin{equation}\label{eq:pseudoinverse_regret}
\frac{1}{2}\sum_{t=1}^T \bg_t^\top A_t^{\dagger} \ \bg_t  +  2(CL)^2 \eta_t~. 
\end{equation}
The key now is to bound the term $\sum_{t=1}^T \bg_t^\top \Ahat_t^{\dagger} \ \bg_t$
where we define $\Ahat_t = \sum_{s=1}^t \bg_s \bg_s^\top$.  In order
to do this, we proceed similarly to the proof
of~\citep[Theorem~4.2]{CesabianchiCoGe05} to show that this term is of
order $\scO(d^2\ln T)$ in the worst case.

\begin{theorem}\label{thm:pseudoinverse}
Let $\lambda^*$ be the minimum among the smallest nonzero eigenvalues of $\Ahat_t \; (t = 1, \ldots, T)$
and $r$ be the rank of $\Ahat_T$.
We have
\[ \sum_{t=1}^T \bg_t^\top \Ahat_t^{\dagger} \ \bg_t \leq r + 
\frac{(1+r)r}{2} \ln \left(1 + \frac{2 \sum_{t = 1}^T \norm{\bg_t}^2_2}{(1+r)r\lambda^*}  \right)~.
  \]
\end{theorem}
\begin{proof}
First by \citet[Lemma~D.1]{CesabianchiCoGe05}, we have
\[
    \bg_t^\top \Ahat_t^{\dagger} \ \bg_t
= 
    \left\{ \begin{array}{cl}
        1 & \text{if $\bg_t \notin \range(\Ahat_{t-1})$}
    \\
        1 - \frac{\mydet(\Ahat_{t-1})}{\mydet(\Ahat_{t})} < 1 & \text{if $\bg_t \in \range(\Ahat_{t-1})$}
    \end{array} \right.
\]
where $\mydet(M)$ denotes the product of the nonzero eigenvalues of matrix $M$.  
We thus separate the steps $t$ such that $\bg_t \in \range(\Ahat_{t-1})$ from those where $\bg_t \notin \range(\Ahat_{t-1})$.
For each $k=1,\dots,r$ let $T_k$ be the first time step $t$ in which the rank of $A_t$ is $k$ (so that $T_1=1$). Also let $T_{r+1} = T+1$ for convenience. With this notation, we have
\begin{align*}
    \sum_{t=1}^T \bg_{t}^{\top} \Ahat_t^{\dagger} \ \bg_{t} 
&= 
    \sum_{k=1}^r \left( \bg_{T_k}^{\top} \Ahat_{T_k}^{\dagger} \bg_{T_k}
+ 
    \sum_{t = T_k+1}^{T_{k+1}-1} \bg_{t}^{\top} \Ahat_t^{\dagger} \ \bg_{t}\right)
\\ &=
    \sum_{k = 1}^r \left(1 + \sum_{t = T_k+1}^{T_{k+1}-1} \left(1-\frac{\mydet(\Ahat_{t-1})}{\mydet(\Ahat_t)}\right) \right)
\\&=
    r + \sum_{k = 1}^r \sum_{t = T_k+1}^{T_{k+1}-1} \left(1-\frac{\mydet(\Ahat_{t-1})}{\mydet(\Ahat_t)}\right)
\\ &\le
    r + \sum_{k = 1}^r \sum_{t = T_k+1}^{T_{k+1}-1} \ln \frac{\mydet(\Ahat_t)}{\mydet(\Ahat_{t-1})}
\\&=
    r + \sum_{k = 1}^r \ln \frac{\mydet(\Ahat_{T_{k+1}-1})}{\mydet(\Ahat_{T_k})}~.
\end{align*}
Fix any $k$ and let $\lambda_{k,1},\dots,\lambda_{k,k}$ be the nonzero eigenvalues of $\Ahat_{T_k}$ and $\lambda_{k,1}+\mu_{k,1},\dots,\lambda_{k,k}+\mu_{k,k}$ be the nonzero eigenvalues of $\Ahat_{T_{k+1}-1}$. Then 
\begin{align*}
    \ln \frac{\mydet(\Ahat_{T_{k+1}-1})}{\mydet(\Ahat_{T_k})}
=
    \ln \prod_{i=1}^k \frac{\lambda_{k,i}+\mu_{k,i}}{\lambda_{k,i}}
=
    \sum_{i=1}^k \ln \left(1+\frac{\mu_{k,i}}{\lambda_{k,i}}\right)~.
\end{align*}
Hence, we arrive at
\begin{align*}
    \sum_{t=1}^T \bg_{t}^{\top} \Ahat_t^{+}\bg_{t}
\le
    r + \sum_{k = 1}^r \sum_{i=1}^k \ln \left(1+\frac{\mu_{k,i}}{\lambda_{k,i}}\right)~.
\end{align*}
To further bound the latter quantity, we use $\lambda^* \leq \lambda_{k,i}$ and Jensen's inequality :
\begin{align*}
    \sum_{k = 1}^r \sum_{i=1}^k \ln \left(1+\frac{\mu_{k,i}}{\lambda_{k,i}}\right) 
&\leq  
    \sum_{k = 1}^r \sum_{i=1}^k \ln \left(1+\frac{\mu_{k,i}}{\lambda^*}\right) 
\\&\leq 
    \frac{(1+r)r}{2} \ln \left(1 + \frac{2 \sum_{k = 1}^r \sum_{i=1}^k \mu_{k,i}}{(1+r)r\lambda^*}  \right)~.
\end{align*}
Finally noticing that 
\[ \sum_{i=1}^k \mu_{k,i} = \trace{\Ahat_{T_{k+1}-1}} - \trace{\Ahat_{T_k}} 
= \sum_{t = T_k + 1}^{T_{k+1}-1} \trace{\bg_t \bg_t^\top} = \sum_{t = T_k + 1}^{T_{k+1}-1} \norm{\bg_t}^2_2
\]
completes the proof.
\end{proof}
Taken together, Eq.~\eqref{eq:pseudoinverse_regret} and Theorem~\ref{thm:pseudoinverse} lead to the following regret bounds (recall the definitions of $\lambda^*$ and $r$ from Theorem~\ref{thm:pseudoinverse}).
\begin{cor}
If $\sigma_t = 0$ for all $t$ and $\eta_t$ is set to be $\frac{1}{CL}\sqrt{\frac{d}{t}}$, then the regret of the algorithm defined by Eq.~\eqref{eq:pseudo_AON} is at most
\[
\frac{CL}{2}\sqrt{\frac{T}{d}}  \left(r + 
\frac{(1+r)r}{2} \ln \left(1 + \frac{2 \sum_{t = 1}^T \norm{\bg_t}^2_2}{(1+r)r\lambda^*}  \right)\right) + 4CL\sqrt{Td} .
\]
On the other hand, if $\sigma_t \geq \sigma > 0$ for all $t$ and $\eta_t$ is set to be $0$, then the regret is at most
\[  \frac{1}{2\sigma} \left(r + 
\frac{(1+r)r}{2} \ln \left(1 + \frac{2 \sum_{t = 1}^T \norm{\bg_t}^2_2}{(1+r)r\lambda^*}  \right)\right)~.
\]
\end{cor}

\section{Proof of Theorem~\ref{thm:FD}}\label{app:FD}
\begin{proof}
We again first apply Proposition~\ref{prop:meta} 
(recall the notation $R_G$ and $R_D$ stated in the proposition).
By the construction of the sketch, we have 
\[A_t - A_{t-1} = S_t^\top
S_t - S_{t-1}^\top S_{t-1} = \bghat_t\bghat_t^\top - \rho_t V_t^\top
V_t \preceq \bghat_t\bghat_t^\top.\]  

It follows immediately that $R_D$ is again at most $4(CL)^2
\sum_{t=1}^T \eta_t$.  For the term $R_G$, we will apply the following
guarantee of Frequent Directions (see the proof of Theorem~1.1 of
\citep{GhashamiLiPhWo15}): $ \sum_{t=1}^T \rho_t \leq
\frac{\Omega_k}{m - k}. $ Specifically, since $\trace{V_t A_t^{-1}
  V_t^\top} \leq \frac{1}{\alpha}\trace{V_t V_t^\top} =
\frac{m}{\alpha}$ we have
\begin{align*}
R_G &= \sum_{t=1}^T \frac{1}{\sigma_t + \eta_t} \inner{A_t^{-1}, A_t -
  A_{t-1} + \rho_t V_t^\top V_t} \\ 
&\leq  \frac{1}{\sigma + \eta_T} \sum_{t=1}^T  \left( \inner{A_t^{-1},
  A_t - A_{t-1} + \rho_t V_t^\top V_t} \right) \\ 
&=  \frac{1}{\sigma + \eta_T} \sum_{t=1}^T  \left( \inner{A_t^{-1},
  A_t - A_{t-1}} +  \rho_t \trace{V_t A_t^{-1} V_t^\top} \right) \\ 
&\leq \frac{1}{(\sigma + \eta_T)} \sum_{t=1}^T  \inner{A_t^{-1}, A_t -
  A_{t-1}}  + \frac{m\Omega_k}{(m-k)(\sigma+\eta_T)\alpha}~.  
\end{align*}
Finally for the term $\sum_{t=1}^T \inner{A_t^{-1}, A_t - A_{t-1}}$,
we proceed similarly to the proof of Theorem~\ref{thm:AON}:
\begin{align*}
\sum_{t=1}^T  \inner{A_t^{-1}, A_t - A_{t-1}} &\leq \sum_{t=1}^T \ln \frac{|A_t|}{|A_{t-1}|} 
= \ln \frac{|A_T|}{|A_{0}|} = \sum_{i=1}^d \ln \left(1 +
\frac{\lambda_i(S_T^\top S_T)}{\alpha} \right) \\ 
&= \sum_{i=1}^m \ln \left(1 + \frac{\lambda_i(S_T^\top S_T)}{\alpha} \right) 
\leq m  \ln\left(1 + \frac{\trace{S_T^\top S_T}}{m\alpha}\right)  
\end{align*}
where the first inequality is by the concavity of the function
$\ln|X|$, the second one is by Jensen's inequality, and the last
equality is by the fact that $S_T^\top S_T$ is of rank $m$ and thus
$\lambda_i(S_T^\top S_T) = 0$ for any $i > m$.  This concludes the
proof.
\end{proof}

\section{Sparse updates for FD sketch}
\label{app:sparse}

The sparse version of our algorithm with the Frequent Directions
option is much more involved.  We begin by taking a detour and
introducing a fast and epoch-based variant of the Frequent Directions
algorithm proposed in~\citep{GhashamiLiPhWo15}.  The idea is the
following: instead of doing an eigendecomposition immediately after
inserting a new $\bghat$ every round, we double the size of the sketch
(to $2m$), keep up to $m$ recent $\bghat$'s, do the decomposition only
at the end of every $m$ rounds and finally keep the top $m$
eigenvectors with shrunk eigenvalues.  The advantage of this variant
is that it can be implemented straightforwardly in $\scO(md)$ time on
average without doing a complicated rank-one SVD update, while still
ensuring the exact same guarantee with the only price of doubling the
sketch size.

Algorithm~\ref{alg:FD_epoch} shows the details of this variant and how
we maintain $H$.  The sketch $S$ is always represented by two parts:
the top part ($DV$) comes from the last eigendecomposition, and the bottom
part ($G$) collects the recent to-sketch vector $\bghat$'s.  Note that
within each epoch, the update of $H^{-1}$ is a rank-two update and
thus $H$ can be updated efficiently using Woodbury formula
(Lines~\ref{line:update_H_1} and~\ref{line:update_H_2} of
Algorithm~\ref{alg:FD_epoch}).

\begin{algorithm}[h!]
\caption{Frequent Direction Sketch (epoch version)}
\label{alg:FD_epoch}
\begin{algorithmic}[1]

\Internal $\tau, D, V, G$ and $H$. 

\vspace{5pt}
\Init{$\alpha, m$}
\State Set $\tau = 1, D = \bzero_{m\times m}, G = \bzero_{m\times d}, H = \tfrac{1}{\alpha} \eye{2m}$ and let $V$ be any $m \times d$ matrix whose rows are orthonormal.
\State Return $(\bzero_{2m \times d}, H)$.

\vspace{5pt}
\setcounter{ALG@line}{0}
\Update{$\bghat$}
    \State Insert $\bghat$ into the $\tau$-th row of $G$.	

    \If{$\tau < m$}
	
	\State Let $\be$ be the $2m \times 1$ basis vector whose $(m + \tau)$-th entry is 1 and $\bq = S\bghat - \tfrac{\bghat^\top\bghat}{2}\be$. \label{line:update_H_1}
	\State Update $H \leftarrow H - \frac{H \bq \be^\top H}{1 + \be^\top H \bq}$ and $H \leftarrow H - \frac{H \be \bq^\top H}{1 + \bq^\top H\be}$. \label{line:update_H_2}
         \State Update $\tau \leftarrow \tau + 1$.
    \Else
        \State $(V, \Sigma) \leftarrow \textbf{ComputeEigenSystem}\left(\left( \begin{array}{c}  DV \\ G \end{array} \right)\right)$ (Algorithm~\ref{alg:eigen}). \label{alg:FD:eigen}
        \State Set $D$ to be a diagonal matrix with $D_{i,i} = \sqrt{\Sigma_{i,i} - \Sigma_{m, m}}, \; \forall i \in [m]$.        
        \State Set $H \leftarrow \diag{\frac{1}{\alpha + D_{1,1}^2},  \cdots, \frac{1}{\alpha + D_{m,m}^2}, \frac{1}{\alpha}, \ldots, \frac{1}{\alpha}} $.        
        \State Set $ G = \bzero_{m \times d}$.
        \State Set $\tau = 1$.
    \EndIf
    
    \State Return $\left(\left( \begin{array}{c}  DV \\ G \end{array} \right), H\right) $   .
\end{algorithmic}
\end{algorithm}

Although we can use any available algorithm that runs in $\scO(m^2 d)$
time to do the eigendecomposition (Line~\ref{alg:FD:eigen} in
Algorithm~\ref{alg:FD_epoch}), we explicitly write down the procedure
of reducing this problem to eigendecomposing a small square matrix in
Algorithm~\ref{alg:eigen}, which will be
important for deriving the sparse version of the algorithm.
Lemma~\ref{lem:eigen} proves that
Algorithm~\ref{alg:eigen} works correctly for finding the top $m$
eigenvector and eigenvalues.

\begin{algorithm}[h!]
\caption{ComputeEigenSystem$(S)$}
\label{alg:eigen}
\begin{algorithmic}[1]
\Require $S = \left( \begin{array}{c}  DV \\ G \end{array} \right)$.
\Ensure $V' \in \R^{m \times d}$  and diagonal matrix $\Sigma \in \R^{m \times m}$ such that the $i$-th row of $V'$ and the $i$-th entry  of the diagonal of $\Sigma$ are the $i$-th eigenvector and eigenvalue of $S^\top S$ respectively.
\State Compute $M = GV^\top$.
\State Decompose $G - MV$ into the form $LQ$ where $L \in \R^{m \times r}$, $Q$ is a $r \times d$ matrix whose rows are orthonormal and $r$ is the rank of $G - MV$ (e.g. by a Gram-Schmidt process). \label{alg:eigen:LQ}
\State Compute the top $m$ eigenvectors ($U \in \R^{m \times (m + r)}$) and eigenvalues ($\Sigma \in \R^{m \times m }$) of the matrix
$\left( \begin{array}{cc}  D^2 & \bzero_{m \times r} \\ \bzero_{r \times m} & \bzero_{r \times r} \end{array} \right) +  \left( \begin{array}{c}  M^\top \\ L^\top \end{array} \right) \left( \begin{array}{cc}  M & L \end{array} \right) $.
\State Return $(V', \Sigma)$ where $V' = U\left( \begin{array}{c}  V \\ Q \end{array} \right)$.
\end{algorithmic}
\end{algorithm}

\begin{lemma}\label{lem:eigen}
The outputs of Algorithm~\ref{alg:eigen} are such that the $i$-th row of $V'$ and the $i$-th entry  of the diagonal of $\Sigma$ are the $i$-th eigenvector and eigenvalue of $S^\top S$ respectively.
\end{lemma}
\begin{proof}
Let $W^\top \in \R^{d \times (d - m - r) }$ be an orthonormal basis of the null space of $\left(\begin{array}{c} V \\  Q \end{array}\right)$.
By Line~\ref{alg:eigen:LQ}, we know that $GW^\top = \bzero$ and $E = (V^\top \;  Q^\top \; W^\top)$ forms an orthonormal basis of $\R^d$. 
Therefore, we have
\begin{align*}
S^\top S &= V^\top D^2 V + G^\top G  \\
&= E \left(\begin{array}{ccc} D^2 & \bzero & \bzero \\ \bzero & \bzero & \bzero \\ \bzero & \bzero & \bzero \end{array}\right) E^\top + E E^\top G^\top G E E^\top \\
&= E \left(  \left(\begin{array}{ccc} D^2 & \bzero & \bzero \\ \bzero & \bzero & \bzero \\ \bzero & \bzero & \bzero \end{array}\right)
+   \left(\begin{array}{c} VG^\top \\ QG^\top \\ WG^\top \end{array}\right) (GV^\top \; G Q^\top \; G W^\top)    \right) E^\top \\
&=  (V^\top \; Q^\top) \underbrace{\left( \left( \begin{array}{cc}  D^2 & \bzero \\ \bzero & \bzero \end{array} \right) +  \left( \begin{array}{c}  M^\top \\ L^\top \end{array} \right) \left( \begin{array}{cc}  M & L \end{array} \right)\right) }_{= C}  
 \left(\begin{array}{c} V \\ Q \end{array}\right)
\end{align*}
where in the last step we use the fact $GQ^\top = (MV + LQ)Q^\top = L$. 
Now it is clear that the eigenvalue of $C$ will be the eigenvalue of $S^\top S$ and the eigenvector of $C$ will be the eigenvector of $S^\top S$ after left multiplied by matrix $(V^\top \; Q^\top)$, completing the proof.
\end{proof}

We are now ready to present the sparse version of \alg with Frequent Direction sketch
(Algorithm~\ref{alg:SFDN}).  The key point is that we represent $V_t$ as
$F_t Z_t$ for some $F_t \in \R^{m \times m}$ and $Z_t \in \R^{m \times d}$, and
the weight vector $\bw_t$ as $\bbw_t + Z_{t-1}^\top \bb_t$ and ensure that the update of 
$Z_t$ and $\bbw_t$ will always be sparse. 
To see this, denote the sketch $S_t$ by $\left( \begin{array}{c}  D_t F_t Z_t
  \\ G_t \end{array} \right) $ and let $H_{t,1}$ and $H_{t,2}$ be the top and bottom half of $H_t$.
Now the update rule of $\bu_{t+1}$ can be rewritten as
\begin{align*}
\bu_{t+1} &= \bw_t - \big(\eye{d} - S_t^\top H_t S_t\big)\tfrac{\bg_t}{\alpha}  \\
&= \bbw_t + Z_{t-1}^\top \bb_t - \frac{1}{\alpha} \bg_t + \frac{1}{\alpha} (Z_t^\top F_t^\top D_t, G_t^\top) \left( \begin{array}{c}  H_{t,1} S_t \bg_t
  \\ H_{t,2} S_t \bg_t \end{array} \right) \\
  &= \underbrace{\bbw_t + \frac{1}{\alpha} (G_t^\top H_{t,2} S_t \bg_t - \bg_t) - (Z_t - Z_{t-1})^\top \bb_t}_{\bbu_{t+1}}
  + Z_t^\top\underbrace{ (\bb_t + \frac{1}{\alpha}F_t^\top D_t H_{t,1}S_t \bg_t ) }_{\bb'_{t+1}}
\end{align*}
We will show that $Z_t - Z_{t-1} = \Delta_t G_t$ for some $\Delta_t \in \R^{m \times m}$ shortly, and thus
the above update is efficient due to the fact that the rows of $G_t$ are collections of previous sparse vectors $\bghat$.

Similarly, the update of $\bw_{t+1}$ can be written as
\begin{align*}
\bw_{t+1} &= \bu_{t+1} - \scale_t (\bx_{t+1} - S_t^\top H_t S_t \bx_{t+1}) \\
&=  \bbu_{t+1} + Z_t^\top \bb'_{t+1} - \scale_t \bx_{t+1} + \scale_t (Z_t^\top F_t^\top D_t, G_t^\top) \left( \begin{array}{c}  H_{t,1} S_t \bx_{t+1}
  \\ H_{t,2} S_t \bx_{t+1} \end{array} \right)  \\
&= \underbrace{\bbu_{t+1} + \scale_t  (G_t^\top H_{t,2}S_t \bx_{t+1}  - \bx_{t+1})}_{\bbw_{t+1}} + Z_t^\top \underbrace{(\bb'_{t+1} + \scale_t F_t^\top D_t H_{t,1} S_t \bx_{t+1})}_{\bb_{t+1}} .
\end{align*}
It is clear that $\scale_t$ can be computed efficiently, and thus the update of $\bw_{t+1}$ is also efficient.
These updates correspond to Line~\ref{alg:SFDN:projection} and~\ref{alg:SFDN:weight_update} of Algorithm~\ref{alg:SFDN}. 

It remains to perform the sketch update efficiently.
Algorithm~\ref{alg:SFD} is the sparse version of Algorithm~\ref{alg:FD_epoch}.
The challenging part is to compute eigenvectors and eigenvalues efficiently.
Fortunately, in light of Algorithm~\ref{alg:eigen}, using the new representation $V = FZ$
one can directly translate the process to Algorithm~\ref{alg:sparse_eigen}
and find that the eigenvectors can be expressed in the form $N_1 Z + N_2 G$.
To see this, first note that Line 1 of both algorithms compute the same matrix $M = GV^\top = GZ^\top F^\top$.
Then Line~\ref{alg:eigen:decompose} decomposes the matrix
\[
G - MV = G - MFZ = \left(\begin{array} {cc} -MF & \eye{m} \end{array} \right) \left( \begin{array}{c}  Z \\ G \end{array} \right)
\defeq PR
\]
using Gram-Schmidt into the form $LQR$ such that the rows of $QR$ are orthonormal 
(that is, $QR$ corresponds to $Q$ in Algorithm~\ref{alg:eigen}).
While directly applying Gram-Schmidt to $PR$ would take $\scO(m^2 d)$ time, 
this step can in fact be efficiently implemented by performing Gram-Schmidt to $P$ (instead of $PR$) in a Banach space 
where inner product is defined as $\langle \ba, \bb \rangle = \ba^\top K \bb$ 
with 
\[ 
K = RR^\top = \left(\begin{array} {cc} ZZ^\top & ZG^\top \\ GZ^\top & GG^\top \end{array}\right)
\] 
being the Gram matrix of $R$.
Since we can efficiently maintain the Gram matrix of $Z$ (see Line~\ref{alg:SFDN:Gram} of Algorithm~\ref{alg:SFD}) 
and $GZ^\top$ and $GG^\top$ can be computed sparsely,
this decomposing step can be done efficiently too.
This modified Gram-Schmidt algorithm is presented in Algorithm~\ref{alg:Gram-Schmidt}
(which will also be used in sparse Oja's sketch),
where Line~\ref{alg:Gram-Schmidt:inner} is the key difference compared to standard Gram-Schmidt
(see Lemma~\ref{lem:Gram-Schmidt} below for a formal proof of correctness).

Line 3 of Algorithms~\ref{alg:eigen} and~\ref{alg:sparse_eigen} are exactly the same.
Finally the eigenvectors $U\left( \begin{array}{c}  V \\ Q \end{array} \right)$ in Algorithm~\ref{alg:eigen} now becomes
(with $U_1, U_2, Q_1, Q_2, N_1, N_2$ defined in Line 4 of Algorithm~\ref{alg:sparse_eigen})
\begin{align*}
U\left( \begin{array}{c}  FZ \\ QR \end{array} \right) &= (U_1, U_2) \left( \begin{array}{c}  FZ \\ QR \end{array} \right)
= U_1 FZ + U_2 (Q_1, Q_2) \left( \begin{array}{c}  Z \\ G \end{array} \right) \\
&= (U_1 FZ + U_2 Q_1) Z + U_2 Q_2 G = N_1 Z + N_2 G.
\end{align*}

Therefore, having the eigenvectors in the form $N_1 Z + N_2 G$, we can simply update $F$ as $N_1$
and $Z$ as $Z + N_1^{-1} N_2 G$ so that the invariant $V = FZ$ still holds 
(see Line~\ref{alg:SFDN:Z_update} of Algorithm~\ref{alg:SFD}).
The update of $Z$ is sparse since $G$ is sparse.

We finally summarize the results of this section in the following theorem.
\begin{theorem}
The average running time of Algorithm~\ref{alg:SFDN} is $\scO\bigl(m^2
+ ms\bigr)$ per round, and the regret bound is exactly the same as the
one stated in Theorem~\ref{thm:FD}.
\end{theorem}

\begin{algorithm}[t]
\caption{Sparse \alglong with Frequent Directions}
\label{alg:SFDN}
\begin{algorithmic}[1]
\Require Parameters $C$, $\alpha$ and $m$.
\State Initialize $\bbu = \bzero_{d \times 1}$, $\bb = \bzero_{m \times 1}$ and $(D, F, Z, G, H) \leftarrow \textbf{SketchInit}(\alpha, m)$ (Algorithm~\ref{alg:SFD}).
\State Let $S$ denote the matrix $\left( \begin{array}{c}  DFZ
  \\ G \end{array} \right)$ throughout the algorithm (without actually
computing it). 
\State Let $H_1$ and $H_2$ denote the upper and lower half of $H$,
i.e. $H = \left( \begin{array}{c}  H_1 \\ H_2 \end{array} \right)$. 
\For{$t=1$ {\bfseries to} $T$}

    \State Receive example $\bx_{t}$.
    \State Projection step: compute $\bxhat = S\bx_t$ and $\scale =
    \frac{\tau_C(\bbu^\top \bx_{t} + \bb^\top Z\bx_t)}{\bx_{t}^\top
      \bx_{t} - \bxhat^\top H \bxhat}$. 
      
      Obtain $\bbw = \bbu + \scale (G^\top H_2 \bxhat - \bx_t )$ and $\bb \leftarrow  \bb + \scale F^\top DH_1
    \bxhat$.  \label{alg:SFDN:projection} 
    \State Predict label $y_t = \bbw^\top \bx_t + \bb^\top Z\bx_t$ and
    suffer loss $\ell_t(y_t)$. 
    \State Compute gradient $\bg_t = \ell_t'(y_t) \bx_t$ and the
    to-sketch vector $\bghat = \sqrt{\sigma_t + \eta_t}\bg_t$. 
    \State $(D, F, Z, G, H, \Delta) \leftarrow \textbf{SketchUpdate}(\bghat)$ (Algorithm~\ref{alg:SFD}).
    \State Update $\bbu = \bbw + \frac{1}{\alpha}(G^\top H_2 S \bg - \bg) - G^\top\Delta^\top \bb$ and  
               $\bb \leftarrow \bb + \frac{1}{\alpha} F^\top DH_1 S\bg $. \label{alg:SFDN:weight_update}
\EndFor
\end{algorithmic}
\end{algorithm}

\begin{algorithm}[t]
\caption{Sparse Frequent Direction Sketch}
\label{alg:SFD}
\begin{algorithmic}[1]
\Internal $\tau, D, F, Z, G, H$ and $K$. 

\vspace{5pt}
\Init{$\alpha, m$}
\State Set $\tau = 1, D = \bzero_{m\times m}, F = K = \eye{m}, H = \tfrac{1}{\alpha} \eye{2m}, G = \bzero_{m\times d}$, and let $Z$ be any $m \times d$ matrix whose rows are orthonormal.
\State Return $(D, F, Z, G, H)$.

\vspace{5pt}
\setcounter{ALG@line}{0}
\Update{$\bghat$}
    \State Insert $\bghat$ into the $\tau$-th row of $G$.	

    \If{$\tau < m$}
	
	\State Let $\be$ be the $2m \times 1$ basic vector whose $(m + \tau)$-th entry is 1 and compute $\bq = S\bghat - \tfrac{\bghat^\top\bghat}{2}\be$.
	\State Update $H \leftarrow H - \frac{H \bq \be^\top H}{1 + \be^\top H \bq}$ and $H \leftarrow H - \frac{H \be \bq^\top H}{1 + \bq^\top H\be}$.                  
	\State Set $\Delta = \bzero_{m \times m}$.
         \State Set $\tau \leftarrow \tau + 1$.

    \Else
    
        \State $(N_1, N_2, \Sigma) \leftarrow \textbf{ComputeSparseEigenSystem} \left(\left( \begin{array}{c}  DFZ \\ G \end{array} \right), K\right)$ (Algorithm~\ref{alg:sparse_eigen}).
        \State Compute $\Delta = N_1^{-1} N_2$. 
        \State Update Gram matrix $K \leftarrow K + \Delta G Z^\top + ZG^\top \Delta^\top + \Delta G G^\top \Delta^\top $.  \label{alg:SFDN:Gram}
        \State Update $F = N_1, Z \leftarrow Z + \Delta G$, and let $D$ be such that $D_{i,i} = \sqrt{\Sigma_{i,i} - \Sigma_{m, m}}, \; \forall i \in [m]$. \label{alg:SFDN:Z_update}       
        \State Set $H \leftarrow \diag{\frac{1}{\alpha + D_{1,1}^2},  \cdots, \frac{1}{\alpha + D_{m,m}^2}, \frac{1}{\alpha}, \ldots, \frac{1}{\alpha}} $.        
        \State Set $ G = \bzero_{m \times d}$.
        \State Set $\tau = 1$.
    \EndIf	
    \State Return $(D, F, Z, G, H, \Delta)$.
\end{algorithmic}
\end{algorithm}

\begin{algorithm}[t]
\caption{ComputeSparseEigenSystem$(S, K)$}
\label{alg:sparse_eigen}
\begin{algorithmic}[1]
\Require $S = \left( \begin{array}{c}  DFZ \\ G \end{array} \right)$ and Gram matrix $K = ZZ^\top$.
\Ensure $N_1, N_2 \in \R^{m \times m}$  and diagonal matrix
$\Sigma \in \R^{m \times m}$ such that the $i$-th row of $N_1 Z + N_2
G$ and the $i$-th entry  of the diagonal of $\Sigma$ are the $i$-th
eigenvector and eigenvalue of the matrix $S^\top S$. 
\State Compute $M = GZ^\top F^\top$.
\State $(L, Q) \leftarrow
\text{Decompose}\left(\left(\begin{array} {cc} -MF &
  \eye{m} \end{array} \right), \left(\begin{array} {cc} K & ZG^\top
  \\ GZ^\top & GG^\top \end{array}\right) \right)$
(Algorithm~\ref{alg:Gram-Schmidt}).  \label{alg:eigen:decompose}
\State Let $r$ be the number of columns of $L$. Compute the top $m$ eigenvectors ($U \in \R^{m \times (m + r)}$) and eigenvalues ($\Sigma \in \R^{m \times m }$) of the matrix
$\left( \begin{array}{cc}  D^2 & \bzero_{m \times r} \\ \bzero_{r
    \times m} & \bzero_{r \times r} \end{array} \right) +
\left( \begin{array}{c}  M^\top \\ L^\top \end{array} \right)
\left( \begin{array}{cc}  M & L \end{array} \right) $. 
\State Set $N_1 = U_1 F + U_2 Q_1$ and $N_2 = U_2 Q_2$
where $U_1$ and $U_2$ are the first $m$ and last $r$ columns of $U$
respectively, and $Q_1$ and $Q_2$ are the left and right half of $Q$
respectively. 
\State Return $(N_1, N_2, \Sigma)$. 
\end{algorithmic}
\end{algorithm}

\begin{lemma}\label{lem:Gram-Schmidt}
The output of Algorithm~\ref{alg:Gram-Schmidt} ensures that $LQR = PR$ and
the rows of $QR$ are orthonormal. 
\end{lemma}
\begin{proof}
It suffices to prove that Algorithm~\ref{alg:Gram-Schmidt} is exactly the same as using the standard Gram-Schmidt 
to decompose the matrix $PR$ into $L$ and an orthonormal matrix which can be written as $QR$.
First note that when $K = \eye{n}$, Algorithm~\ref{alg:Gram-Schmidt} is simply the standard Gram-Schmidt algorithm applied to $P$.
We will thus go through Line 1-10 of Algorithm~\ref{alg:Gram-Schmidt} with $P$ replaced by $PR$ and $K$ by $\eye{n}$ 
and show that it leads to the exact same calculations as running Algorithm~\ref{alg:Gram-Schmidt} directly.
For clarity, we add ``$\;\tilde{}\;$'' to symbols to distinguish the two cases (so $\tilde{P} = PR$ and $\tilde{K} = \eye{n}$).
We will inductively prove the invariance $\tilde{Q}  = QR$ and $\tilde{L} = L$.
The base case $\tilde{Q}  = QR = \bzero$ and $\tilde{L} = L = \bzero$ is trivial.
Now assume it holds for iteration $i - 1$ and consider iteration $i$. We have
\[
\tilde{\balpha} = \tilde{Q}\tilde{K}\tilde{\bp} = QRR^\top\bp = QK\bp = \balpha,
\]
\[
\tilde{\bbeta} = \tilde{\bp} - \tilde{Q}^\top \tilde{\balpha} = R^\top\bp - (QR)^\top\balpha 
= R^\top (\bp - Q^\top\balpha) = R^\top \bbeta,
\]
\[
\tilde{c} = \sqrt{\tilde{\bbeta}^\top \tilde{K} \tilde{\bbeta}} = \sqrt{(R^\top \bbeta)^\top (R^\top \bbeta)} 
=  \sqrt{\bbeta^\top K \bbeta} = c,
\]
which clearly implies that after execution of Line 5-9, we again have $\tilde{Q}  = QR$ and $\tilde{L} = L$, finishing the induction.
\end{proof}

\section{Details for sparse Oja's algorithm}\label{app:sparse-oja}

We finally provide the missing details for the sparse version of the
Oja's algorithm. Since we already discussed the updates for $\bbw_t$
and $\bb_t$ in Section~\ref{sec:sparse}, we just need to describe how
the updates for $F_t$ and $Z_t$ work. Recall that the dense Oja's updates
can be written in terms of $F$ and $Z$ as
\begin{equation}
\label{eq:Oja}
\begin{split}
\Lambda_{t} &= (\eye{m} - \Gamma_t) \Lambda_{t-1} + \Gamma_t \;\diag{F_{t-1}Z_{t-1} \bghat_t}^2 \\
F_t Z_t &\xleftarrow{\text{orth}} F_{t-1}Z_{t-1} + \Gamma_t F_{t-1}Z_{t-1} \bghat_t \bghat_t^\top
= F_{t-1} (Z_{t-1} +  F_{t-1}^{-1}\Gamma_t F_{t-1}Z_{t-1} \bghat_t \bghat_t^\top)~.
\end{split}
\end{equation}
Here, the update for the eigenvalues is straightforward.  For the
update of eigenvectors, first we let $Z_t = Z_{t-1} + \bdelta_t
\bghat_t^\top$ where $\bdelta_t = F_{t-1}^{-1}\Gamma_t F_{t-1}Z_{t-1}
\bghat_t$ (note that under the assumption of
Footnote~\ref{fn:full_rank}, $F_t$ is always invertible).  Now it is
clear that $Z_t - Z_{t-1}$ is a sparse rank-one matrix and the update
of $\bbu_{t+1}$ is efficient.  Finally it remains to update $F_t$ so
that $F_t Z_t$ is the same as orthonormalizing $F_{t-1} Z_t$, which
can in fact be achieved by applying the Gram-Schmidt algorithm to
$F_{t-1}$ in a Banach space where inner product is defined as $\langle
\ba, \bb \rangle = \ba^\top K_t \bb$ where $K_t$ is the Gram matrix
$Z_t Z_t^\top$ (see Algorithm~\ref{alg:Gram-Schmidt}).  Since we can
maintain $K_t$ efficiently based on the update of $Z_t$:
\[
K_t = K_{t-1} + \bdelta_t \bghat_t^\top Z_{t-1}^\top + Z_{t-1}
\bghat_t \bdelta_t^\top + (\bghat_t^\top \bghat_t) \bdelta_t
\bdelta_t^\top,  
\]
the update of $F_t$ can therefore be implemented in $\scO(m^3)$ time.

\begin{algorithm}[t]
\caption{Decompose(P, K)}
\label{alg:Gram-Schmidt}
\begin{algorithmic}[1]
\Require $P \in \R^{m \times n}$,  $K \in \R^{m\times m}$ such that $K$
is the Gram matrix $K = RR^\top$ for some matrix $R \in \R^{n \times
  d}$ where $n \geq m, d \geq m$, 
\Ensure $L \in \R^{m \times r}$ and $Q \in \R^{r \times n}$
such that $LQR = PR$ where $r$ is the rank of $PR$ and
the rows of $QR$ are orthonormal. 
\State Initialize $L = \bzero_{m \times m}$ and $Q = \bzero_{m \times n}$.
\For{$i=1$ {\bfseries to} $m$}
    \State Let $\bp^\top$ be the $i$-th row of $P$.
    \State Compute $\balpha = Q K \bp, \bbeta = \bp - Q^\top \balpha$ and $c = \sqrt{\bbeta^\top K \bbeta}$. \label{alg:Gram-Schmidt:inner}
    \If{$c \neq 0$}
    	\State Insert $\frac{1}{c}\bbeta^\top$ to the $i$-th row of $Q$.
    \EndIf
    \State Set the $i$-th entry of $\balpha$ to be $c$ and insert $\balpha$ to the $i$-th row of $L$.
\EndFor
\State Delete the all-zero columns of $L$ and all-zero rows of $Q$.
\State Return $(L, Q)$.
\end{algorithmic}
\end{algorithm}

\section{Experiment Details}\label{app:experiment}
This section reports some detailed experimental results omitted from Section~\ref{subsec:real_data}.
Table~\ref{tab:datasets} includes the description of benchmark datasets;
Table~\ref{tab:error2} reports error rates on relatively small datasets to show that \ojaalg generally has better performance;
Table~\ref{tab:error} reports concrete error rates for the experiments described in Section~\ref{subsec:real_data};
finally Table~\ref{tab:eigen} shows that Oja's algorithm estimates the eigenvalues accurately. 

As mentioned in Section~\ref{subsec:real_data}, we see substantial improvement for the {\it splice} dataset
when using Oja's sketch even after the diagonal adaptation. 
We verify that the condition number for this dataset before and after the diagonal adaptation are very close 
(682 and 668 respectively), explaining why a large improvement is seen using Oja's sketch.
Fig.~\ref{fig:splice} shows the decrease of error rates as \ojaalg with different sketch sizes sees more examples.
One can see that even with $m=1$ \ojaalg already performs very well.
This also matches our expectation since there is a huge gap between the top and second eigenvalues 
of this dataset ($50.7$ and $0.4$ respectively).

\input{table_data}

\input{table_err2}

\input{table_err}

\input{table_eigen}

\begin{figure}[t]
\centering
 \includegraphics[width=.6\textwidth]{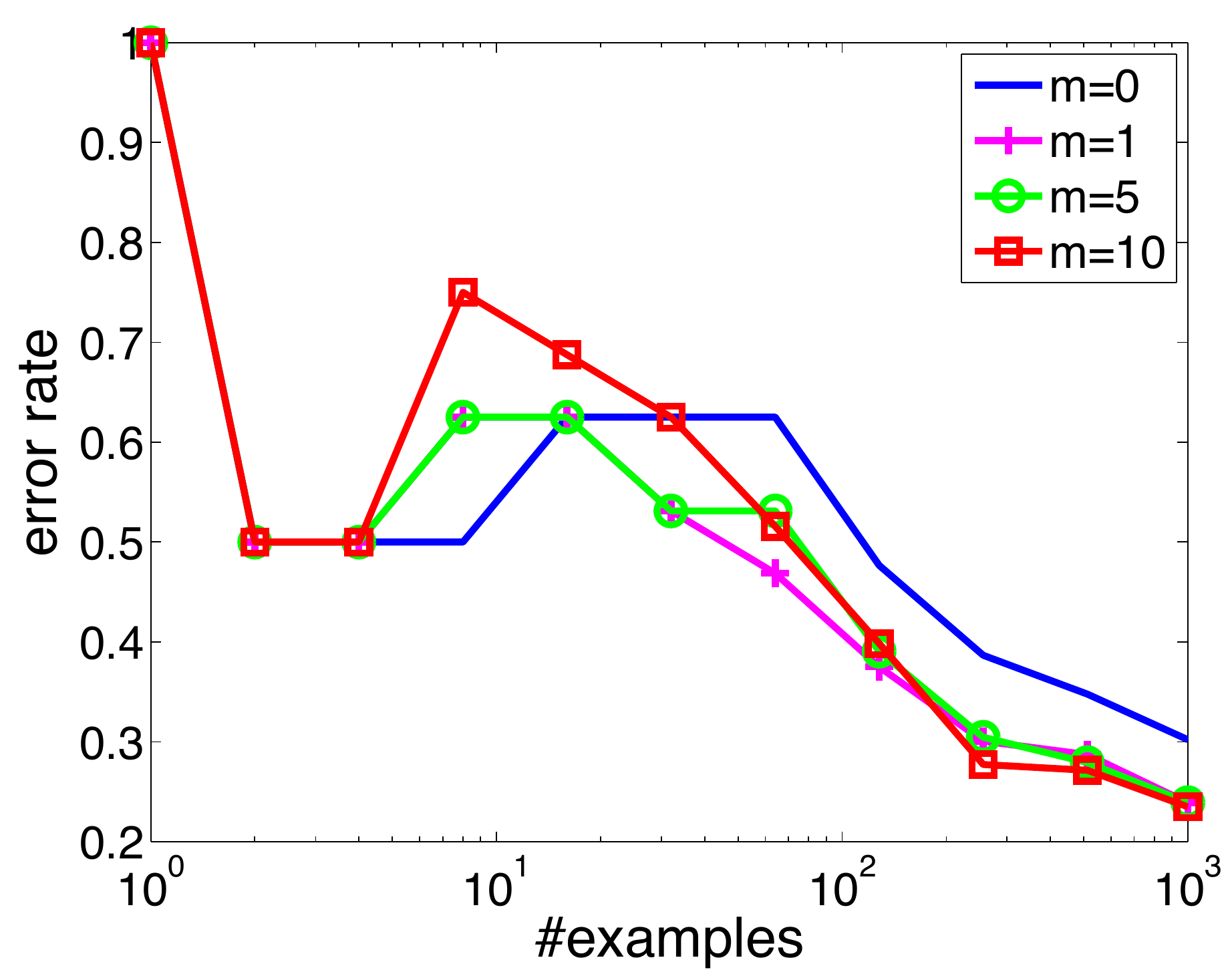}
\caption{Error rates for \ojaalg with different sketch sizes on splice dataset}
\label{fig:splice}
\end{figure}

\end{document}

%% file: table_data.tex
\begin{table}[t] 
\caption{Datasets used in experiments}\label{tab:datasets}
\vskip 10pt
\begin{center}
\begin{tabular}{|c||r|r|r|}
\hline
Dataset & \#examples & avg. sparsity & \#features \\
\hline
\hline
20news &  18845  &  93.89  & 101631  
\\
a9a &  48841  &  13.87  & 123  
\\
acoustic & 78823 & 50.00 & 50
\\
adult &  48842  &  12.00  & 105  
\\
australian &  690 & 11.19 & 14
\\
breast-cancer & 683 & 10.00 & 10
\\
census &  299284  &  32.01  & 401  
\\
cod-rna & 271617 & 8.00 & 8
\\
covtype &  581011  &  11.88  & 54  
\\
diabetes & 768 & 7.01 & 8 
\\
gisette & 1000 & 4971.00 & 5000
\\
heart & 270 & 9.76 & 13
\\
ijcnn1 & 91701 & 13.00 & 22
\\
ionosphere & 351 & 30.06 & 34
\\
letter &  20000  &  15.58  & 16  
\\
magic04 & 19020 & 9.99 & 10
\\
mnist & 11791 & 142.43 & 780
\\
mushrooms &  8124 & 21.00 & 112
\\
rcv1&  781265  &  75.72  & 43001  
\\
real-sim & 72309 & 51.30 & 20958
\\
splice & 1000 & 60.00 & 60
\\
w1a & 2477 & 11.47 & 300
\\
w8a & 49749 &11.65   &  300 
\\
\hline
\end{tabular}
\end{center}
\end{table}

%% file: table_err2.tex
\begin{table}[t] 
\caption{Error rates for Sketched Online Newton with different sketching algorithms}\label{tab:error2}
\vskip 10pt
\begin{center}
\begin{tabular}{|c||r|r|}
\hline
Dataset &  FD-SON & Oja-SON \\
\hline
australian &  16.0 & 15.8 \\
\hline
breast-cancer  & 5.3 & \bf{3.7} \\
\hline
diabetes  & 35.4 & \bf{32.8} \\
\hline
mushrooms  & 0.5 & \bf{0.2} \\
\hline
splice   & \bf{22.6} & 22.9 \\
\hline
\end{tabular}
\end{center}
\end{table}

%% file: table_err.tex
\begin{table}[t] 
\caption{Error rates for different algorithms (with best results bolded)}\label{tab:error}
\vskip 10pt
\begin{center}
\begin{tabular}{|c||c|c|c|c|c|}
\hline
\multirow{3}{*}{Dataset} &\multicolumn{4}{c|}{\ojaalg}  &  \multirow{3}{*}{\textsc{AdaGrad}} \\
\cline{2-5}
& \multicolumn{2}{c|}{Without Diagonal Adaptation} & \multicolumn{2}{c|}{With Diagonal Adaptation} & \\
\cline{2-5}
& $m=0$ & $m=10$ & $m=0$ & $m=10$ & \\
\hline
\hline
20news &  0.121338 & 0.121338 & \bf{0.049590} & \bf{0.049590} & 0.068020
\\
a9a &  0.204447  &  0.195203  & \bf{0.155953} & \bf{0.155953} & 0.156414
\\
acoustic & 0.305824 & 0.260241 & \bf{0.257894} & \bf{0.257894} & 0.259493
\\
adult &  0.199763  &  0.199803  & \bf{0.150830}  & \bf{0.150830} & 0.181582
\\
australian &  0.366667 & 0.366667 & 0.162319 & \bf{0.157971} & 0.289855
\\
breast-cancer & 0.374817 & 0.374817 & \bf{0.036603} & \bf{0.036603} & 0.358712
\\
census &  0.093610  &  0.062038  & 0.051479  & \bf{0.051439} & 0.069629
\\
cod-rna & 0.175107 & 0.175107 & 0.049710 & \bf{0.049643} & 0.081066
\\
covtype &  \bf{0.042304}  &  \bf{0.042312}  & 0.050827  & 0.050818 & 0.045507
\\
diabetes & 0.433594 & 0.433594 & 0.329427 & \bf{0.328125} & 0.391927
\\
gisette & 0.208000 & 0.208000 & \bf{0.152000} & \bf{0.152000} & 0.154000
\\
heart & 0.477778 & 0.388889 & \bf{0.244444} & \bf{0.244444} & 0.362963
\\
ijcnn1 & 0.046826 & 0.046826 & \bf{0.034536} & \bf{0.034645} & 0.036913
\\
ionosphere & 0.188034 & \bf{0.148148} & 0.182336 & 0.182336 & 0.190883
\\
letter &  0.306650  &  0.232300  & 0.233250  & \bf{0.230450} & 0.237350
\\
magic04 & 0.000263 & 0.000263 & \bf{0.000158} & \bf{0.000158} & 0.000210
\\
mnist & 0.062336 & 0.062336 & 0.040031& \bf{0.039182} & 0.046561
\\
mushrooms &  0.003323 & 0.002339 & 0.002462 & 0.002462 & \bf{0.001969}
\\
rcv1&  0.055976  &  0.052694  & 0.052764  & 0.052766 & \bf{0.050938}
\\
real-sim & 0.045140 & 0.043577 & \bf{0.029498} & \bf{0.029498} & 0.031670
\\
splice & 0.343000 & 0.343000 & 0.294000 & \bf{0.229000} & 0.301000
\\
w1a & 0.001615 & 0.001615 & 0.004845 & 0.004845 & \bf{0.003633}
\\
w8a & \bf{0.000101} & \bf{0.000101}  &  0.000422  & 0.000422 & 0.000221
\\
\hline
\end{tabular}
\end{center}
\end{table}

%% file: table_eigen.tex
\begin{table}[t] 
\caption{Largest relative error between true and estimated top 10
  eigenvalues using Oja's rule.}\label{tab:eigen}
\vskip 10pt
\begin{center}
\begin{tabular}{|c|r|}
\hline
Dataset & \specialcell{Relative eigenvalue \\ difference} \\
\hline
\hline
a9a & 0.90 \\
australian & 0.85 \\
breast-cancer & 5.38 \\
diabetes & 5.13 \\
heart & 4.36 \\
ijcnn1 &  0.57 \\
magic04 & 11.48 \\
mushrooms & 0.91 \\
splice & 8.23 \\
w8a & 0.95 \\
\hline
\end{tabular}
\end{center}
\end{table}

%% file: sketchedOL.bbl
\begin{thebibliography}{33}
\providecommand{\natexlab}[1]{#1}
\providecommand{\url}[1]{\texttt{#1}}
\expandafter\ifx\csname urlstyle\endcsname\relax
  \providecommand{\doi}[1]{doi: #1}\else
  \providecommand{\doi}{doi: \begingroup \urlstyle{rm}\Url}\fi

\bibitem[Balsubramani et~al.(2013)Balsubramani, Dasgupta, and
  Freund]{BalsubramaniDaFr13}
A.~Balsubramani, S.~Dasgupta, and Y.~Freund.
\newblock The fast convergence of incremental pca.
\newblock In \emph{NIPS}, 2013.

\bibitem[Byrd et~al.(2016)Byrd, Hansen, Nocedal, and Singer]{ByrdHaNoSi14}
R.~H. Byrd, S.~Hansen, J.~Nocedal, and Y.~Singer.
\newblock A stochastic quasi-newton method for large-scale optimization.
\newblock \emph{SIAM Journal on Optimization}, 26:\penalty0 1008--1031, 2016.

\bibitem[Cesa-Bianchi and Lugosi(2006)]{CesabianchiLu06}
N.~Cesa-Bianchi and G.~Lugosi.
\newblock \emph{Prediction, Learning, and Games}.
\newblock Cambridge University Press, 2006.

\bibitem[Cesa-Bianchi et~al.(2005)Cesa-Bianchi, Conconi, and
  Gentile]{CesabianchiCoGe05}
N.~Cesa-Bianchi, A.~Conconi, and C.~Gentile.
\newblock A second-order perceptron algorithm.
\newblock \emph{SIAM Journal on Computing}, 34\penalty0 (3):\penalty0 640--668,
  2005.

\bibitem[Duchi et~al.(2011)Duchi, Hazan, and Singer]{DuchiHaSi2011}
J.~Duchi, E.~Hazan, and Y.~Singer.
\newblock Adaptive subgradient methods for online learning and stochastic
  optimization.
\newblock \emph{JMLR}, 12:\penalty0 2121--2159, 2011.

\bibitem[Erdogdu and Montanari(2015)]{ErdogduMo15}
M.~A. Erdogdu and A.~Montanari.
\newblock Convergence rates of sub-sampled newton methods.
\newblock In \emph{NIPS}, 2015.

\bibitem[Frank and Wolfe(1956)]{FrankWo56}
M.~Frank and P.~Wolfe.
\newblock An algorithm for quadratic programming.
\newblock \emph{Naval research logistics quarterly}, 3\penalty0 (1-2):\penalty0
  95--110, 1956.

\bibitem[Gao et~al.(2013)Gao, Jin, Zhu, and Zhou]{GaoJiZhZh13}
W.~Gao, R.~Jin, S.~Zhu, and Z.-H. Zhou.
\newblock One-pass auc optimization.
\newblock In \emph{ICML}, 2013.

\bibitem[Garber and Hazan(2016)]{GarberHa13}
D.~Garber and E.~Hazan.
\newblock A linearly convergent conditional gradient algorithm with
  applications to online and stochastic optimization.
\newblock \emph{SIAM Journal on Optimization}, 26:\penalty0 1493--1528, 2016.

\bibitem[Garber et~al.(2015)Garber, Hazan, and Ma]{garber2015online}
D.~Garber, E.~Hazan, and T.~Ma.
\newblock Online learning of eigenvectors.
\newblock In \emph{ICML}, 2015.

\bibitem[Ghashami et~al.(2015)Ghashami, Liberty, Phillips, and
  Woodruff]{GhashamiLiPhWo15}
M.~Ghashami, E.~Liberty, J.~M. Phillips, and D.~P. Woodruff.
\newblock Frequent directions: Simple and deterministic matrix sketching.
\newblock \emph{SIAM Journal on Computing}, 45:\penalty0 1762--1792, 2015.

\bibitem[Ghashami et~al.(2016)Ghashami, Liberty, and Phillips]{GhashamiLiPh16}
M.~Ghashami, E.~Liberty, and J.~M. Phillips.
\newblock Efficient frequent directions algorithm for sparse matrices.
\newblock In \emph{KDD}, 2016.

\bibitem[Gonen and Shalev-Shwartz(2015)]{GonenSh15}
A.~Gonen and S.~Shalev-Shwartz.
\newblock Faster sgd using sketched conditioning.
\newblock \emph{arXiv:1506.02649}, 2015.

\bibitem[Gonen et~al.(2016)Gonen, Orabona, and Shalev-Shwartz]{GonenOrSh16}
A.~Gonen, F.~Orabona, and S.~Shalev-Shwartz.
\newblock Solving ridge regression using sketched preconditioned svrg.
\newblock In \emph{ICML}, 2016.

\bibitem[Hardt and Price(2014)]{HardtPr14}
M.~Hardt and E.~Price.
\newblock The noisy power method: A meta algorithm with applications.
\newblock In \emph{NIPS}, 2014.

\bibitem[Hazan and Kale(2012)]{HazanKa12}
E.~Hazan and S.~Kale.
\newblock Projection-free online learning.
\newblock In \emph{ICML}, 2012.

\bibitem[Hazan et~al.(2007)Hazan, Agarwal, and Kale]{HazanAgKa07}
E.~Hazan, A.~Agarwal, and S.~Kale.
\newblock Logarithmic regret algorithms for online convex optimization.
\newblock \emph{Machine Learning}, 69\penalty0 (2-3):\penalty0 169--192, 2007.

\bibitem[Jaggi(2013)]{Jaggi13}
M.~Jaggi.
\newblock Revisiting frank-wolfe: Projection-free sparse convex optimization.
\newblock In \emph{ICML}, 2013.

\bibitem[Li et~al.(2015)Li, Lin, and Lu]{LiLiLu15}
C.-L. Li, H.-T. Lin, and C.-J. Lu.
\newblock Rivalry of two families of algorithms for memory-restricted streaming
  pca.
\newblock \emph{arXiv:1506.01490}, 2015.

\bibitem[Liberty(2013)]{Liberty13}
E.~Liberty.
\newblock Simple and deterministic matrix sketching.
\newblock In \emph{KDD}, 2013.

\bibitem[Liu and Nocedal(1989)]{LiuNo89}
D.~C. Liu and J.~Nocedal.
\newblock On the limited memory bfgs method for large scale optimization.
\newblock \emph{Mathematical programming}, 45\penalty0 (1-3):\penalty0
  503--528, 1989.

\bibitem[McMahan and Streeter(2010)]{McMahanSt2010}
H.~B. McMahan and M.~Streeter.
\newblock Adaptive bound optimization for online convex optimization.
\newblock In \emph{COLT}, 2010.

\bibitem[Mokhtari and Ribeiro(2015)]{MokhtariRi14}
A.~Mokhtari and A.~Ribeiro.
\newblock Global convergence of online limited memory bfgs.
\newblock \emph{JMLR}, 16:\penalty0 3151--3181, 2015.

\bibitem[Moritz et~al.(2016)Moritz, Nishihara, and Jordan]{MoritzNiJo15}
P.~Moritz, R.~Nishihara, and M.~I. Jordan.
\newblock A linearly-convergent stochastic l-bfgs algorithm.
\newblock In \emph{AISTATS}, 2016.

\bibitem[Oja(1982)]{Oja82}
E.~Oja.
\newblock Simplified neuron model as a principal component analyzer.
\newblock \emph{Journal of mathematical biology}, 15\penalty0 (3):\penalty0
  267--273, 1982.

\bibitem[Oja and Karhunen(1985)]{OjaKa85}
E.~Oja and J.~Karhunen.
\newblock On stochastic approximation of the eigenvectors and eigenvalues of
  the expectation of a random matrix.
\newblock \emph{Journal of mathematical analysis and applications},
  106\penalty0 (1):\penalty0 69--84, 1985.

\bibitem[Orabona and P{\'a}l(2015)]{OrabonaPa15}
F.~Orabona and D.~P{\'a}l.
\newblock Scale-free algorithms for online linear optimization.
\newblock In \emph{ALT}, 2015.

\bibitem[Orabona et~al.(2015)Orabona, Crammer, and Cesa-Bianchi]{orabona2015}
F.~Orabona, K.~Crammer, and N.~Cesa-Bianchi.
\newblock A generalized online mirror descent with applications to
  classification and regression.
\newblock \emph{Machine Learning}, 99\penalty0 (3):\penalty0 411--435, 2015.

\bibitem[Pilanci and Wainwright(2015)]{PilanciWa15}
M.~Pilanci and M.~J. Wainwright.
\newblock Newton sketch: A linear-time optimization algorithm with
  linear-quadratic convergence.
\newblock \emph{arXiv:1505.02250}, 2015.

\bibitem[Ross et~al.(2013)Ross, Mineiro, and Langford]{RossMiLa13}
S.~Ross, P.~Mineiro, and J.~Langford.
\newblock Normalized online learning.
\newblock In \emph{UAI}, 2013.

\bibitem[Schraudolph et~al.(2007)Schraudolph, Yu, and
  G{\"u}nter]{SchraudolphYuGu07}
N.~N. Schraudolph, J.~Yu, and S.~G{\"u}nter.
\newblock A stochastic quasi-newton method for online convex optimization.
\newblock In \emph{AISTATS}, 2007.

\bibitem[Sohl-Dickstein et~al.(2014)Sohl-Dickstein, Poole, and
  Ganguli]{SohldicksteinPoGa14}
J.~Sohl-Dickstein, B.~Poole, and S.~Ganguli.
\newblock Fast large-scale optimization by unifying stochastic gradient and
  quasi-newton methods.
\newblock In \emph{ICML}, 2014.

\bibitem[Woodruff(2014)]{Woodruff14}
D.~P. Woodruff.
\newblock Sketching as a tool for numerical linear algebra.
\newblock \emph{Foundations and Trends in Machine Learning}, 10\penalty0
  (1-2):\penalty0 1--157, 2014.

\end{thebibliography}
